\documentclass{article}

\usepackage[preprint]{neurips_2026}

\usepackage[utf8]{inputenc} 
\usepackage[T1]{fontenc}    
\usepackage{hyperref}       
\usepackage{url}            
\usepackage{booktabs}       
\usepackage{amsfonts}       
\usepackage{nicefrac}       
\usepackage{microtype}      
\usepackage{xcolor}         

\usepackage{graphicx}
\usepackage{subcaption}
\usepackage{amsmath}
\usepackage{amssymb}
\usepackage{mathtools}
\usepackage{amsthm}
\usepackage{braket}
\usepackage{cleveref}
\usepackage{ulem}

\theoremstyle{plain}
\newtheorem{theorem}{Theorem}[section]
\newtheorem{proposition}[theorem]{Proposition}
\newtheorem{lemma}[theorem]{Lemma}
\newtheorem{corollary}[theorem]{Corollary}
\theoremstyle{definition}

\newtheorem{assumption}[theorem]{Assumption}
\theoremstyle{remark}

\title{Action-Factored Multi-Agent Reinforcement Learning for Scalable Quantum Device Tuning}

\author{
  Edwin De Nicolo\thanks{Equal Contribution. Order determined alphabetically.} \\
  Department of Physics \\
  University of Oxford \\
  United Kingdom \\
  \texttt{edwin.denicolo@merton.ox.ac.uk} \\
  \And
  Rahul Marchand\thanks{Equal Contribution. Order determined alphabetically.} \\
  Department of Engineering Science \\
  University of Oxford \\
  United Kingdom \\
  \texttt{rahul.marchand@new.ox.ac.uk} \\
  \And
  Cornelius Carlsson \\
  Department of Engineering Science \\
  University of Oxford \\
  United Kingdom \\
  \texttt{cornelius.carlsson@spc.ox.ac.uk} \\
  \And
  Pranav Vaidhyanathan \\
  Department of Engineering Science \\
  University of Oxford \\
  United Kingdom \\
  \texttt{pranav@robots.ox.ac.uk} \\
  \And
  Natalia Ares \\
  Department of Engineering Science \\
  University of Oxford \\
  United Kingdom \\
  \texttt{natalia.ares@eng.ox.ac.uk} \\
}

\begin{document}

\maketitle

\begin{abstract}
 Cooperative multi-agent reinforcement learning is well suited to problems with large parameter spaces and exploitable local structure, such as the tuning of electrostatically-defined quantum-dot arrays. However, if parameter cross-talk is strong, a non-stationary environment from the perspective of any individual agent can destabilize learning -- the same effect that plagues manual tuning of such systems. We propose using a factored representation of the action space, learned online, to decouple agents and minimize their interference. Our framework, QADAPT, uses this factorization to efficiently learn shared policies based on local measurements and rewards. With this modular strategy, we achieve zero-shot generalization to unseen quantum device sizes and maintain an approximately constant number of convergence steps to reach target regimes. This work provides a scalable route toward the rapid calibration of large-scale quantum processors.
\end{abstract}

\section{Introduction}
\label{intro}

Reinforcement learning (RL) provides a general framework for sequential decision-making under uncertainty, where agents learn policies through interaction with an environment~\cite{sutton2018reinforcement}. Combining RL with deep neural networks has achieved human-level performance on complex, high-dimensional tasks~\cite{mnih2015human}, and its extension to continuous control domains has allowed agents to learn efficiently from sensory inputs~\cite{lillicrap2020continuous}. This opens practical applications for robotics and locomotion, and more recently, has bridged RL with the field of quantum computing for tasks such as error-robust quantum gate design~\cite{baum2021experimental}, analog pulse shaping~\cite{semola2022deep}, real-time qubit feedback~\cite{reuer2023realizing, vaidhyanathan2026quantum}, and voltage tuning of quantum dots~\cite{nguyen2021deep}. While RL, and machine learning in general, is becoming central to quantum control~\cite{alexeev2025artificial, cao2026qcaleval}, a key challenge lies in scaling these computational methods efficiently to larger devices \cite{ares2021machine, zwolak2024data, kondo2025environment, metasym, schorling2025meta, bukov2026reinforcement}. 

Semiconductor spin qubits in particular, while a compelling quantum computing platform \cite{chatterjee2021semiconductor, bartee2025spin, steinacker2025industry, dijkema2026simultaneous}, face enormous control complexity due to measurement noise, device non-uniformities, and cross-talk. The tuning procedure involves calibrating dc gate voltages to create confined islands of charge, i.e. quantum dots, as shown in Fig.~\ref{Fig1}a. There exist two main gate types; \textit{plunger gates} predominantly control the electrochemical potential of quantum dots, hence their charge occupations, and \textit{barrier gates} modulate the tunnel coupling between dots. Because of dense gate electrode arrangements, adjusting one gate voltage perturbs neighboring dots via capacitive cross-talk~\cite{borsoi2024shared}. Measuring the system's response as a function of two gate voltages produces images known as \textit{charge stability diagrams} (CSDs). These diagrams contain relevant tuning and cross-talk information, which may be inferred from the position and curvature of the line features they contain. As shown in Fig.~\ref{Fig1}b, the objective is to reach a voltage configuration that brings all charge occupations and tunnel-couplings to a predefined target.

Even with existing autotuning approaches -- including fully autonomous spin qubit tuning~\cite{schuff2026fully} and a broad class of quantum dot tuning methods~\cite{moon2020machine, zwolak2020autotuning, roux2025rapid, carlsson2025automated} -- demonstrations beyond two-dot (sub)systems are lacking. Sequential tuning of smaller sub-systems offers no guarantee of global optimality, and designing meaningful score functions for supervised or Bayesian algorithms becomes increasingly non-trivial as parameter spaces grow and tuning landscapes become highly non-convex. Since data acquisition remains the primary time bottleneck in practice, stricter sample efficiency is also demanded.

To address these bottlenecks, we present QADAPT (Quantum Arrays of Dots in Automated Parallel Tuning), a multi-agent reinforcement learning (MARL) framework for tuning quantum devices at scale. Our key insight is that learning a factored representation of the voltage parameter space online enables decoupled control over each dots' electrochemical potential. Under this factorization, QADAPT becomes highly amenable to a \textit{centralized training with decentralized execution} paradigm (CTDE), whereby agents not only act independently, conditioned on local observations, but can also benefit from physically-motivated policy sharing strategies. Our resulting modular architecture mitigates the combinatorial growth of action spaces associated with larger devices, enabling parallelization and zero-shot generalization as newly introduced agents can directly inherit existing policies.
 
In summary, our core contributions are:
\begin{itemize} 
    \item {\textbf{Adaptive Action-Space Factorization:} We develop a Kalman-guided pipeline that estimates the local gate-to-dot capacitance matrix online and uses it to construct a virtual action basis. This structure-aware reparameterization reduces cross-agent action interference and improves sample efficiency.}
    \item {\textbf{Role-Shared Decentralized MARL:} We formulate quantum dot tuning as a cooperative continuous Decentralized Partially Observable Markov Decision Processes (Dec-POMDP) and introduce a modular actor-critic architecture with gate-type-specific parameter sharing~\cite{gupta2017cooperative}. This decomposes a high-dimensional global search problem while keeping the number of learned policies independent of array size.}
    \item \textbf{Zero-Shot Scalability and Empirical Evaluation:} We demonstrate that QADAPT zero-shot generalizes across different system sizes, unseen device configurations, and various non-idealized physical effects without retraining. We empirically evaluate our framework against various state-of-the-art baselines and demonstrate superior performance across all benchmarks.  
\end{itemize}

\begin{figure*}[ht]
  \begin{center}
    \centerline{\includegraphics[width=\textwidth]{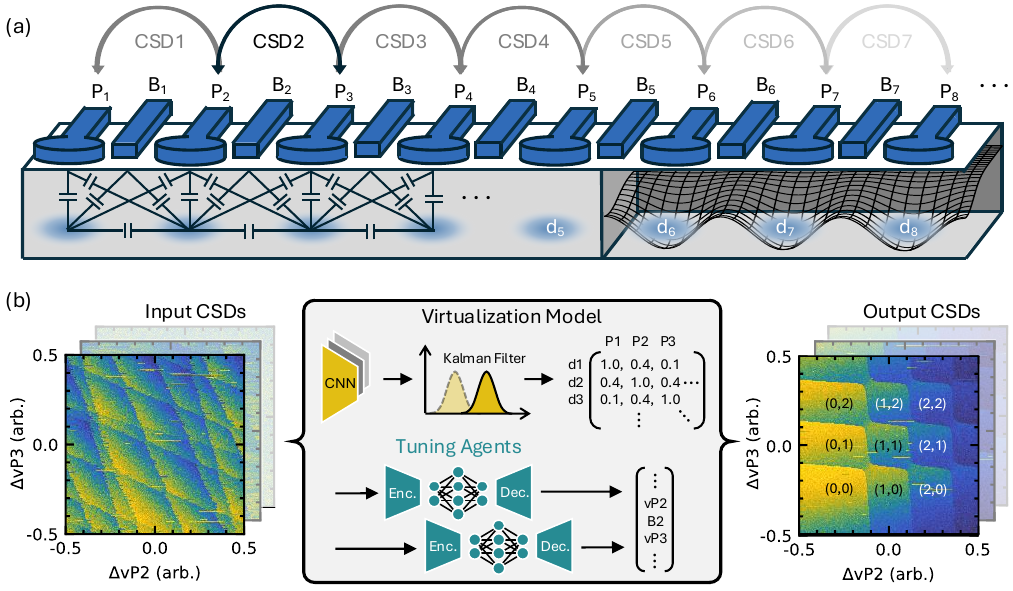}}
    \caption{\textbf{Algorithmic Flow of QADAPT} \textbf{(a)} Barrier gates ($\mathrm{B_x}$) interdigitate plunger gates ($\mathrm{P_x}$), and primarily control dot-dot tunnel couplings and charge occupations, respectively. All gates contribute in defining the confinement potential necessary to form quantum dots ($\mathrm{d_1...d_D}$). Gate cross-talk can be approximated by a capacitive network, which in general extends beyond nearest neighbors, greatly complicating array tuning. Charge stability diagrams (CSDs) are acquired by sweeping the voltage on neighboring pairs of plunger gates and bear information on the array's current tuning state. \textbf{(b)} Each CSD is passed through a virtualization model and tuning agent model. The virtualization model returns relative cross-talk estimates between $\mathrm{P_x}$ and dots $\mathrm{d_{x-1},d_x,d_{x+1}}$, refined by a Kalman filter at each tuning step, resulting in a cross-capacitance matrix that enables decoupled plunger control. Tuning agents operate in the decoupled plunger gate basis, outputting decentralized actions in the form of new virtual voltage setpoints denoted by $\mathrm{vP_x}$. Through collective optimization of local objectives, a global target configuration is reached with a desired level of tunnel-coupling between dots, and charge occupation across all dots in the array (right).}
    \label{Fig1}
  \end{center}
\end{figure*}

\section{Related Work}
\label{related}

Automatic tuning of semiconductor quantum dots spans simple physics-inspired heuristics~\cite{baart2016computer, yon2025experimental} to machine-learning-assisted pipelines \cite{schuff2026fully, carlsson2025automated}. Recent works have introduced various classifiers to enable the autonomous tuning of quantum dots~\cite{kalantre2019machine, ziegler2023tuning, van2025all}, reach target charge occupations~\cite{durrer2020automated, czischek2021miniaturizing}, and calibrate tunnel couplings~\cite{van2020quantum}. In conjunction with Bayesian optimization, classifier methods have also proven successful across multiple device architectures~\cite{severin2024cross}, yet, their validation is restricted beyond double-dot systems. 

The standard strategy for larger systems relies on sequential tuning: capacitive crosstalk is probed iteratively and quantum dots are tuned using virtual voltage controls via an "n+1" procedure ~\cite{volk2019loading}. Using CSDs as input, the underlying feature-tracking task for gate virtualization has been offloaded to vision models, including Convolutional Neural Networks (CNNs)~\cite{tsuzuki2026cnn}, U-Net-like architectures~\cite{rao2025modular, muto2025automatic} and Transformers~\cite{hader2025automated, marchand2025end}, but requires additional pre-processing for larger systems~\cite{oakes2020automatic}. In sequential tuning, however, devices are prone to biasing into locally compatible rather than globally optimal voltage configurations. Furthermore, post-hoc and fixed cross-capacitance estimates greatly limit practical voltage tuning windows.

To cope with the combinatorial growth of joint action spaces, decentralized policies are often preferred; however, the coordination of multiple agents can be challenging when cross-talk is present due to ambiguous credit assignment and conflicting local objectives~\cite{oroojlooy2023review}. The simplest approach -- independent Q-learning (IQL)~\cite{tan1993multi}, where each agent treats others as part of the environment -- generally lacks convergence guarantees due to the non-stationarity induced by simultaneous policy updates. A substantial body of work has addressed these challenges, which we group into five categories:

\textbf{Value function factorization} decomposes the joint action-value function into per-agent utilities, enabling decentralized execution under a shared reward signal. In deep RL, VDN introduced additive value decomposition, later generalized by QMIX through monotonic mixing, and by QTRAN through unconstrained surrogate factorization~\cite{sunehag2017value,rashid2020monotonic,son2019qtran}. COMIX and FacMADDPG extend factorization to the continuous control domain using cross-entropy optimization and factored critics, respectively~\cite{de2020deep}.

\textbf{Centralized critics} augment each agent's value estimate with global information. COMA~\cite{foerster2018counterfactual} introduced a counterfactual baseline to produce agent-specific advantage estimates, directly addressing credit assignment. In the multi-agent setting, MADDPG~\cite{lowe2017multi} conditioned each agent's critic on the joint observations and actions of all agents. FACMAC~\cite{peng2021facmac} combined the factored critics of QMIX with the centralized policy gradients of MADDPG.

\textbf{Learning to communicate} enables agents to selectively exchange information to improve coordination, either through graph-structured policies or learned message-passing protocols~\cite{sukhbaatar2016learning}. Most notably, GPG~\cite{khan2020graph} introduced parametrized swarm policies over a proximity graph, InforMARL~\cite{nayak2023scalable} proposed graph attention over local neighborhoods, and MADDPG-M~\cite{kilinc2018multi} learns explicit communication channels alongside task policies.

\textbf{Consensus-based methods} coordinate agents through shared latent representations. These representations can be formed either from joint observations, such as in CMAT~\cite{zhao2026bridging, wen2022multi} using an auto-regressive Transformer decoder, or from local observations, for instance via contrastive learning~\cite{feng2024hierarchical}.

\textbf{Parameter sharing and specialization} shares policies across agents with similar roles, reducing the number of independently trained networks and improving sample efficiency~\cite{christianos2021scaling}. IPPO and MAPPO~\cite{de2020independent, yu2022surprising} showed that shared PPO policies achieve strong cooperative performance, which HAPPO~\cite{kuba2021trust,zhong2024heterogeneous} extended to the heterogeneous agent case. HyperMARL~\cite{tessera2024hypermarl} further mitigated gradient interference using hypernetworks conditioned on agent identity.

In QADAPT, we take a fundamentally different approach to cooperative MARL. Rather than forcing coordination among coupled agents, we propose to decouple agents altogether by learning a \textit{factored action space representation}. By transforming the joint control problem into an approximately separable sum of local objectives, cooperative outcomes naturally follow (see Appendix \ref{app:proof}.)

\section{Methods}
\label{methods}

\subsection{Problem Formulation and Simulation Setting}
We simulate the quantum dot tuning environment using QArray~\cite{van2024qarray, van2024codebase}, acquiring CSDs by sweeping the voltage on pairs of nearest neighbor plunger gates. The array is considered optimally tuned when all inter-dot tunnel couplings meet a pre-defined target value, and each dot contains exactly one charge, as in Fig.~\ref{Fig1}b (right). These conditions are satisfied over a voltage manifold that constitutes a small (<1\%) fraction~\cite{schuff2026fully} of the total voltage hyper-volume, the center of which we denote as the target voltage configuration $\mathbf{v^*}$. We note that orthogonal pairs of transition lines in each CSD are an indicator of accurate cross-talk compensation, however, they are not a pre-requisite for convergence.

We formulate this tuning process as a Decentralized Partially Observable Markov Decision Process (Dec-POMDP) utilizing $2N-1$ agents~\cite{bernstein2002complexity}, representing one agent for each device gate. 
\begin{itemize}
    \item \textbf{State \& Observations:} The system state $s^t = (\mathbf{v}^t_p, \mathbf{v}^t_b) \in \mathbb{R}^{2N-1}$ is the vector of all plunger and barrier gate voltages. At time step $t$, agent $i$ receives a local observation $o_i^t$ consisting of its current normalized voltage $v_i^t \in [-1,1]$ concatenated with a CNN representation of its local CSD scans.
    \item \textbf{Actions:} Each agent outputs a normalized action $\hat{a}_i \in [-1, 1]$, corresponding to an absolute voltage assignment $a_i^t = v_i^{t+1}$. The transition dynamics over the joint action space $\mathbf{a^t}$ are fully deterministic, given by $P(s^{t+1} | s^t, \mathbf{a}_t) = \delta(s^t - \mathbf{a^t})$.
    \item \textbf{Rewards:} Due to varying \textit{lever arms} across simulated devices, identical voltage changes can induce different physical responses. We apply reward shaping~\cite{hu2020learning} to define a normalized per-agent reward $r_i^t \in [0, 1]$, rescaled according to the agent's hidden lever arm $\alpha_i$ and the voltage tuning range. The agents' goal is to maximize the cumulative return $J = \mathbb{E}_{\pi_\theta} \left[\sum_{t=0}^T \sum_{i=1}^N r_i^t\right]$, which in our cooperative setting is equivalent to optimizing a common global objective~\cite{rashid2020monotonic}.
\end{itemize}

\subsection{Adaptive Gate Virtualization}
To mitigate action interference caused by capacitive cross-talk, QADAPT learns a factored representation of the plunger action space online. At each time step, CSDs are passed through a lightweight CNN, $f_{\mathrm{virt}}(\cdot)$, to estimate the capacitances between each plunger gate and its surrounding dots (up to two neighbors away). 

A Kalman filter incrementally updates these estimates, treating the unknown couplings as latent state variables, and uses new measurements as observations for Bayesian updates~\cite{chen2003bayesian, fang2018nonlinear}. This process yields a time-dependent $D \times D$ cross-capacitance matrix $\Phi_t$, which maps physical plunger voltages $\mathbf{v}_p$ to virtual voltages $\mathbf{u}_p=\Phi_t\mathbf{v}_p$.

In this space, the new control variables $u_i$ act approximately independently by diagonalizing the local coupling between gates and dots. Let $\mathbf{e}_t$ denote the local device-state error and let $C_t$ be the local response Jacobian. Under the virtual action update $\Delta \mathbf{v}_p^t=\Phi_t^{-1}\Delta \mathbf{u}_p^t$, the one-step quadratic model becomes:
\begin{equation}
    \frac{1}{2}\|\mathbf{e}_t+C_t\Phi_t^{-1}\Delta \mathbf{u}_p^t\|_2^2.
\end{equation}
When $C_t\Phi_t^{-1}\approx I$, this objective has small mixed action terms and is close to the separable surrogate $\frac{1}{2}\sum_i(e_{t,i}+\Delta u_{p,i}^t)^2$. This transforms the joint control problem into an approximately separable sum of local objectives, preconditioning the local quadratic control landscape (Appendix~\ref{app:proof}).

\subsection{Modular Actor-Critic Architecture}
To process observations and output actions, we employ a modular advantage actor-critic architecture~\cite{konda1999actor}. To reduce model complexity and enable zero-shot scalability to new arrays~\cite{neumann2022scaling}, parameters $\{\theta, \phi\}$ are shared across all agents of the same physical type~\cite{christianos2021scaling} (i.e., one parameter set for all plunger gates, and another for all barrier gates). 

The network relies on a convolutional encoder $f_{\mathrm{CNN}}(\cdot)$ to extract spatial hierarchies from the CSDs in which gate $i$ participates, denoted as $\mathbf{X}_i^t$. This feature representation is concatenated with a learned linear projection $p(\cdot)$ of the gate's current voltage, such that $o_i^t = f_{\mathrm{CNN}}(\mathbf{X}_i^t) || p(v_i^t)$. For plungers at the edge of the array that participate in only one CSD, the single available scan is duplicated across channels. For inner plungers, the two participating CSDs are stacked along the channel dimension. Barrier gates participate only in CSDs acquired by sweeping the gate's two immediate neighboring plungers, as barrier-induced cross-talk on next-nearest dots is near zero~\cite{rao2025modular}. The network then branches:
\begin{itemize}
    \item \textbf{Actor Head:} Outputs a probability distribution over continuous voltage adjustments via a Gaussian policy $\pi_\theta(a_i^t \mid o_i^t)$.
    \item \textbf{Critic Head:} Produces an estimate of the value function $V_\phi(o_i^t)$ based strictly on local observations, preserving decentralized execution.
\end{itemize}

\subsection{Training Pipeline}
We train QADAPT on a 4-dot system using a modified Proximal Policy Optimization (PPO) method for stable and sample-efficient learning in continuous control tasks~\cite{ppo}. At each training iteration, we roll out multiple parallel episodes using the current policy $\pi_\theta$. Because of our parameter-sharing scheme, experiences from all agents of the same type are aggregated into a single batch. This effectively multiplies our batch size by the number of agents, significantly stabilizing updates.

Because the optimal policy at any given step is always to move the system closer to the target configuration, future rewards offer little additional guidance~\cite{bandit}. Therefore, we model the task as a contextual bandit, eliminating temporal credit assignment (discount factor $\gamma = 0$). Each step is treated independently with dense, immediate feedback, which simplifies the learning problem and handles any residual non-stationarity of the multi-agent environment (supported by our ablation in Table~\ref{tab:ablation}).

After collecting a batch of trajectories, per-agent advantage estimates $\hat{A}_i^t$ are computed, which, without temporal credit assignment, reduces to immediate reward residuals:
\begin{equation}
    \hat{A}_i^t = r_i^t - V_\phi(o_i^t)
\end{equation}

To update the shared policy parameters $\theta$, local advantages are aggregated and the joint PPO clipped surrogate objective is maximized across all agents $i$ and batch time steps $t$:
\begin{equation}
    \mathcal{L}^{\mathrm{CLIP}}(\theta) = \hat{\mathbb{E}}_{t, i} \left[ \min\left( \kappa_i^t(\theta) \hat{A}_i^t, \mathrm{clip}(\kappa_i^t(\theta), 1-\epsilon, 1+\epsilon)\hat{A}_i^t \right) \right]
\end{equation}
where $\kappa_i^t(\theta) = \frac{\pi_\theta(a_i^t \mid o_i^t)}{\pi_{\theta_{\mathrm{old}}}(a_i^t \mid o_i^t)}$ is the probability ratio.

Concurrently, the shared critic parameters $\phi$ are updated by minimizing the aggregated mean squared error between the predicted values and the actual returns (which here equal the immediate rewards $r_i^t$):
\begin{equation}
    \mathcal{L}^{\mathrm{VF}}(\phi) = \hat{\mathbb{E}}_{t, i} \left[ \left( V_\phi(o_i^t) - r_i^t \right)^2 \right]
\end{equation}

By computing these expectations over the joint experience buffer of all homogeneous agents, gradients become averaged across the array. This decentralized PPO algorithm is thus an instance of centralized training with decentralized execution~\cite{xu2018experience}.

\begin{figure}[ht]
  \vskip 0.2in
  \begin{center}
    \centerline{\includegraphics[width=\columnwidth]{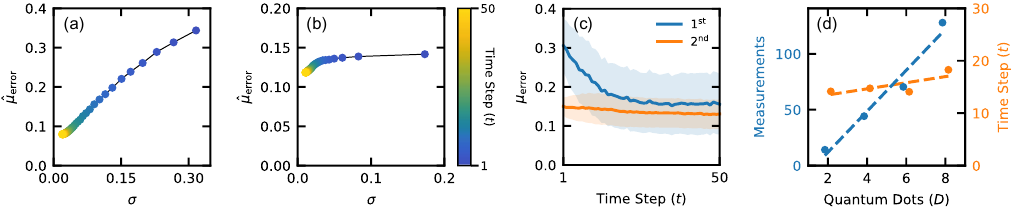}}
    \caption{\textbf{Gate Virtualization and Scaling Performance of QADAPT}. \textbf{a)} Kalman filter's posterior uncertainty (variance) $\sigma$ against absolute estimation error $\hat{\mu}_{\mathrm{error}}$ for nearest neighbor gate-to-dot cross-capacitances. Each marker represents one time step in a 4-dot system (3 CSD measurements), averaged over 1000 episodes. The cross-capacitance has a fixed ground truth value of 0.7. A top-right to bottom-left trajectory demonstrates that capacitance estimates improve over time, and that the filter's uncertainty tracks this improvement. \textbf{b)} same as a), but for second-nearest neighbor gate-to-dot cross-capacitances, with ground-truth values set to 0.3. The trajectory here is weaker but stable, as there is a weaker cross-talk signal from further separated dots. \textbf{c)} Raw prediction errors of $f_\mathrm{CNN}(\cdot)$ over time for both first- and second-nearest gate-to-dot couplings. \textbf{d)} Zero-shot scaling of CSD measurement acquisitions and time steps to reach convergence in 2, 4, 6, and 8 dot systems.}
    \label{Fig2}
  \end{center}
\end{figure}

\begin{figure*}[ht]
  \begin{center}    
  \centerline{\includegraphics[width=\textwidth]{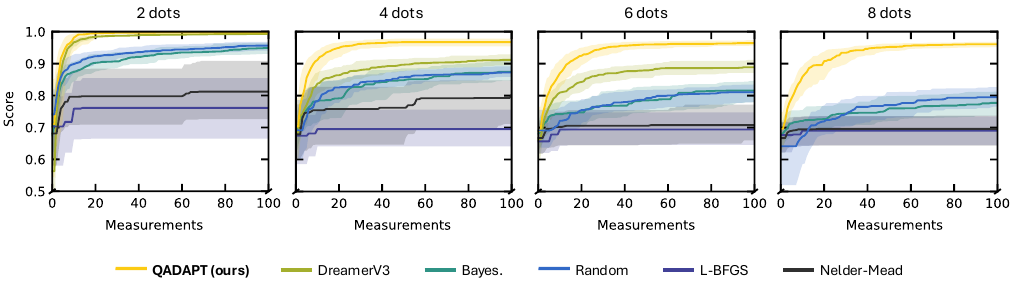}}
    \caption{\textbf{Domain-specific Benchmarks with Array Size}. As the length of the quantum dot array increases from 2 to 8 (moving left to right), QADAPT maintains rapid convergence, out-performing all benchmark methods, including: model-based reinforcement learning (DreamerV3), Bayesian optimisation using Gaussian processes, the gradient free Nelder-Mead optimiser, the gradient-based Limited-memory Broyden-Fletcher-Goldfarb-Shanno algorithm (L-BFGS), and a random search. The score reflects normalized distance to the target voltage configuration across all gates, and is plotted using the cumulative minimum distance at each scan step. Lines represent averages over 100 runs, with an envelope of half a standard deviation. Initial and target voltage configurations are randomized in each run. DreamerV3 is omitted from the 8-dot benchmarks due to compute resource constraints.}
    \label{Fig3}
  \end{center}
\end{figure*}

\section{Experiments}
\label{sec:experiments}

To highlight the benefits of our deep multi-agent approach, we evaluate QADAPT against several strong baseline methods (Fig.~\ref{Fig2}). To represent approaches commonly found in the literature of low-dimensional quantum dot tuning, we first consider several gradient-free optimizers: random search, Nelder--Mead~\cite{nelder1965simplex}, Bayesian optimization with Gaussian processes~\cite{rasmussen2006gpml, shahriari2016taking}, and L-BFGS~\cite{nocedal1980lbfgs} using finite-difference gradient estimates. To represent model-based reinforcement learning, we employ DreamerV3~\cite{dreamer}, adapting its latent world model to plan actions for the tuning task. This comparison probes whether explicit dynamics modeling improves sample efficiency in our measurement-constrained regime.

To identify components responsible for QADAPT’s performance, we conduct several further baseline comparisons and ablations (Table~\ref{tab:ablation}). Firstly, we compare $f_\mathrm{CNN}(\cdot)$ of our actor-critic architecture with three alternative backbones: (i) a heavier ResNet-style encoder (``Nature CNN'')~\cite{he2016resnet,tsuzuki2026cnn}; (ii) an LSTM-based recurrent policy~\cite{hochreiter1997lstm}; and (iii) a vanilla Transformer \cite{vaswani2017attention}. To isolate the contribution of our action space factorization, we ablate the virtualization model in addition to a variant where only the Kalman filtering is disabled. Without virtualization, QADAPT becomes identical to IPPO \cite{de2020independent}. We compare this against MAPPO \cite{yu2022surprising} (the critic receives a concatenation of all agents' observations), allowing us to discern whether giving the critic access to global information during training can overcome the absence of virtualization to coordinate agents. We also include MADDPG \cite{lowe2017multi} and FACMAC \cite{peng2021facmac} as off-policy baselines, probing the benefits of past experiences (in the form of replay buffers) toward agent coordination. Finally, we evaluate the contribution of QADAPT's multi-agent framework using a single-agent PPO baseline. In all comparisons, we report both convergence and steps-to-convergence to disentangle reliability from efficiency.

\begin{table*}[ht]
\caption{\textbf{Ablation Study and MARL-specific Baseline Comparisons}. We report the percentage of runs converging within 100 time steps, and the average number of steps to convergence, $\mu_\mathrm{steps}$, for those runs in a 4-dot array. Each row averages over 100 tuning trials. Convergence is defined as the fraction of runs that terminate with all agents within a threshold radius of their target voltage. Radii of 2\%, 5\%, and 10\% of each agent's total tuning range are reported (see Appendix~\ref{app:params} for more details). $\mu_\mathrm{steps}$ is calculated only over converged runs. Early stopping is used during training and the number of training epochs is shown for each variant}
\label{tab:ablation}
\centering
\small
\setlength{\tabcolsep}{6pt}
\begin{tabular}{lccc}
\toprule
Variant & Convergence Rate (\%) $\uparrow$ & $\mu_{\mathrm{steps}}$ $\downarrow$ & Epochs \\
\midrule
\textbf{QADAPT (ours)} & \textbf{93 / 95 / 95} & \textbf{20.5 / 16.7 / 13.2} & 150 \\
\midrule
w/ Nature CNN \cite{he2016resnet, tsuzuki2026cnn} & 91 / 98 / 100 & 20.1 / 17.8 / 14.6 & 250 \\
w/ LSTM \cite{hochreiter1997lstm} & 73 / 85 / 94 & 22.1 / 17.3 / 12.7 & 275 \\
w/ Transformer \cite{vaswani2017attention} & 3 / 26 / 61 & 33.7 / 29.2 / 27.1 & 150 \\
w/o Kalman filter & 10 / 55 / 89 & 26.9 / 24.2 / 18.1 & 100 \\
$\gamma = 0.9$ & 89 / 94 / 97 & 19.5 / 16.3 / 13.3 & 150 \\
\midrule
IPPO \cite{de2020independent} & 63 / 84 / 98 & 28.0 / 22.5 / 15.2 & 275 \\
MAPPO \cite{yu2022surprising} & 78 / 90 / 95 & 26.1 / 19.2 / 14.2 & 400 \\
MADDPG \cite{lowe2017multi} & 1 / 2 / 5 & 19.0 / 20.5 / 18.8 & 500k$^*$ \\
FACMAC \cite{peng2021facmac} & 0 / 0 / 1 & N/A / N/A / 1 & 500k$^*$ \\
Single Agent PPO & 0 / 1 / 6 & N/A / 43.0 / 26.2 & 150 \\
\bottomrule
\multicolumn{4}{l}{\footnotesize{$^*$For off-policy methods, the number of environment steps is reported.}}
\end{tabular}
\end{table*}

\section{Results and Discussion}
Across all evaluated array sizes and randomized device parameter settings, QADAPT consistently outperforms the learning-based and domain-specific black-box baselines described in Section~\ref{sec:experiments}. In particular, QADAPT achieves the highest convergence rates within a fixed measurement budget while also requiring fewer measurement cycles to reach the target configuration when it converges. 

In Fig.~\ref{Fig2}(a-b), we first illustrate the performance of our virtualization model in the baseline 4-dot system. As tuning progresses, a trend emerges in which the mean error in predicted gate-to-dot cross capacitance decreases, along with the variance in those predictions. This correlation is expected, because the confidence in $\Phi$ -- derived from the Kalman filter's uncertainty signal -- should increase as prediction errors decrease (proof provided in Appendix~\ref{app:proof}). A high density of predictions with small errors, occurring beyond $t\approx20$, is comparable to QADAPT's mean steps to convergence (see Fig.~\ref{Fig3}a), indicating a constructive feedback loop between $\Phi$ and the MDP. Overall, this supports our claim that a factored action space improves agent coordination and therefore accelerates quantum dot array tuning. 

This is further evidenced by our main model for comparison, IPPO, whose performance is shown in Table~\ref{tab:ablation}: removing virtualization leads to a marked drop in convergence and requires more steps to converge. This is consistent with the intuition that cross-talk produces strong action interference in the raw voltage basis, effectively turning local gate voltage updates into global perturbations. Ablating the Kalman filter alone also harms performance, which verifies that the filtering acts as a good regularizer over capacitance estimates, throttling sudden capacitance changes that would otherwise destabilize tuning. This may occur, for example, near target configurations where transition lines in CSDs are sparse. We note that capacitance predictions can be negative, which would correspond to an overcompensation by nearby plunger gate voltages when gate $i$ is updated. We explicitly provide such examples when training $f_{\mathrm{virt}}(\cdot)$ to enable $\Phi$ to self-correct online.

In Fig.~\ref{Fig3}, we present the evolution of the tuning score with the number of simulated measurements (CSD acquisitions). Because data acquisition is typically the dominant wall-clock cost in experiments, we report performance as a function of the number of measured CSDs, rather than gradient steps or environment transitions. The score is defined as $1-\sum_{i=1}^N\tilde{v_i}/N$, where $\tilde{v_i}$ and is the `best-so-far' distance to the target voltage $v_i^*$, normalized by the voltage bounds of the tuning environment. We observe the largest performance margins between QADAPT and traditional L-BFGS and Nelder-Mead optimizers, likely due to the non-convex optimization landscape, which causes the algorithms to become trapped in local minima. For this reason, we are unsurprised to find a random search to be more efficient, and showing higher mean scores. Using RL-based approaches avoids spending many evaluations exploring uninformative directions in the high-dimensional voltage space.

To quantify scalability, we evaluate the learned policy from the 4-dot system on arrays ranging from 2 to 8 dots. In particular, we find that QADAPT’s parameter-sharing strategy across agents of the same type (barrier and plunger) mitigates the combinatorial growth in joint action space that limits centralized baselines, such as DreamerV3. In practice, this manifests as a comparatively shallow degradation in QADAPT's convergence rate as the number of dots increases, which is especially important given the constrained measurement budget of quantum dot tuning. This is exemplified in Fig.~\ref{Fig2}d, where we show QADAPT's near-linear scaling $\mathcal{O}(N)$ in the number of measurements required to reach a mean tuning score of 0.95, which translates to a near constant number of time steps to reach convergence. Another popular domain-specific method, Bayesian optimization, suffers greatly in this context, showing a catastrophic slow-down in convergence beyond 4 dots (see Fig.~\ref{Fig3}). This poses an additional overhead compared to other learning-based and gradient-free methods due to the cubic complexity with observations that Gaussian process surrogate models experience.

MADDPG and FACMAC both fail to converge under our measurement budget (Table 1). We observed actor gradient collapse in MADDPG and critic gradient explosion in FACMAC, suggesting that off-policy replay and centralized value estimation are not in themselves sufficient to stabilize learning when agents act in a strongly coupled action basis. For the single agent variant of our PPO method, we find zero convergence for our strictest convergence threshold criterion. We attribute this to a noisy reward signal stemming from an inability to accredit advantage estimates to specific control dimensions during training. Taken together, we conclude that our proposed factorization of the action space, enabling effective parameter sharing and decentralized execution, is the primary driver of both tuning efficiency and reliability.

In the ablations of our actor-critic architecture, we find that this mainly trades computation for modest changes in steps-to-convergence. The policy is not strongly bottlenecked by visual feature capacity once the essential features in CSDs are captured. Hence, our encoder proves to be effective as a lightweight network as compared to that of larger attention based or recurrent neural networks. 

\subsection{Scalability}
The efficient zero-shot scaling to various array sizes demonstrates how QADAPT's modular training architecture is amenable to adding agents at test time, circumventing credit-assignment, long-horizon planning that potentially compounds errors over imagined roll-outs, and optimization complexity with the growth of a joint action space. 

Thanks to our action space factorization and shared-policy formulation, a QADAPT policy trained on an array with $N$ dots can be directly deployed to an array with $N'$ dots (using $2N'-1$ agents) without architectural changes or retraining. This inherently makes QADAPT parallelizable: by exploiting the locality of each agent, all agents can operate concurrently on their respective part of the tuning task without interference, both during training and execution. This is not the case for single-agent models, which experience an exponential growth in the joint state-action space grows exponentially with the addition of more control variables. For example, training DreamerV3 on 2/4/6 dot-arrays in the benchmark comparison of Fig.~\ref{Fig2} required 2/1/0.5 million training rollouts.

Transferability to various device sizes is a key limitation of other baseline MARL approaches. In FACMAC, despite a factored critic as well as mixing networks applied on agent utilities, policy gradients and the semantics of the value function are still tied to the agent cardinality deployed during training. Similarly in MAPPO, transferability is not guaranteed due to the centralized critic using a concatenation of all agent observations, which is unique to the training environment. We point out that in more scalable graph-based communication strategies, such as InforMARL, agents still only control single entities. In our context, this would equate to each agent acting on a single voltage parameter $\Delta v_p^t$, and the learning of virtualized voltage adjustments $\Delta u_p^t$ would have to take place implicitly. In QADAPT, this becomes entirely off-loaded to our lightweight $f_\mathrm{virt}(\cdot)$, giving more capacity to the actor-critic network for learning optimal actions, reducing training overheads. 

\subsection{Generalizability}
We consider how QADAPT would generalize to two- and three-dimensional arrays, where quantum dots may have more than two nearest neighbors. In this case, QADAPT scales trivially: each gate agent will simply participate in more CSDs and the cross-capacitance matrix $\Phi$ will become more dense, owing to an increased number of nearest- and second-nearest dot neighbors. While this would require re-training $f_{\mathrm{virt}}(\cdot)$, all components of the QADAPT framework would remain unchanged and the algorithm becomes backward compatible. Compute memory can be easily sized according to resource availability and the resolution of CSD images can be adjusted to accommodate faster measurement cycles.

Here, we redirect the reader Appendix~\ref{app:super} where we apply the QADAPT framework to a completely different setting -- the tuning of superconducting qubits in the presence of drive cross-talk. This shares many similarities to the quantum dot array tuning task: each qubit has configurable parameters and the adjustment of any one parameter influences the performance of nearby qubits. With minor modifications to input and output dimensions, we show that an action space factorization and paramater sharing strategy is also highly suitable for this quantum control problem.

\subsection{Limitations}
While QADAPT generalizes well to realistic experimentally simulated data (see Appendix~\ref{app:params}), the sim-2-real gap still remains unexplored. Real hardware can experience parameter drift, hysteresis, and time-varying voltage offsets. While the bandit reduction avoids long-horizon planning that would be destabilized by such effects, they may nonetheless hamper the online learning of $\Phi$ if cross-capacitance values vary significantly between tuning steps. This can be mitigated by increasing the Kalman filter's process noise, which relaxes posterior confidence and permits larger matrix updates, though potentially at the cost of additional steps to convergence.

A further limitation of QADAPT may stem from the strict assignment of agent types among barrier and plunger categories. In simulation, these model distinctly different behaviors, however, fabrication imperfections and gate dimensions may blur their respective roles on real hardware. Whether this function-based parameter sharing strategy degrades gracefully when physical symmetries break remains to be validated. This could be mitigated, for instance, using parameter sharing masks \cite{li2024kaleidoscope}, which could be selected dynamically based on cross-capacitance behaviors.

Finally, we note that simulated devices are assumed to operate in an `open' regime, whereby each quantum dot is capable of acquiring charges from a nearby reservoir. While such device architectures are standard, future devices may operate in an `isolated' regime \cite{bertrand2015quantum}. This introduces non-steady state solutions and modifies charge stability diagram (CSD) features, which may reduce the effectiveness of our non-temporal credit assignment.

\section{Conclusion}
We introduced QADAPT, a multi-agent reinforcement learning framework for the automated tuning of quantum dot arrays that combines (i) parameter sharing across gate-local agents, and (ii) an online factorization of the control space, refined via Kalman-guided cross-capacitance estimation. This formulation enables scalable and decentralized control while preserving coordinated optimization in the tuning of strongly coupled quantum devices. Across simulated quantum dot arrays of increasing size, QADAPT consistently outperforms state-of-the-art RL baselines and black-box optimizers, achieving higher convergence rates under fixed measurement budgets and requiring fewer measurements to reach target configurations. Crucially, its modular design supports zero-shot transfer.

These results suggest that action-factored MARL is a practical and scalable paradigm for quantum device calibration, addressing a key bottleneck in the deployment of large-scale quantum dot-based processors. Future work will focus on hardware-in-the-loop validation and robustness to experimental non-idealities.

\begin{ack}
The authors would like to thank Marios Samiotis for helpful discussions. C.C. acknowledges support from the UKRI Doctoral Training Partnership related to EP/W524311/1 (project ref. 2887634). N.A. acknowledges support from the European Research Council (Grant Agreement No. 948932) and the Royal Society (URF-R1-191150). The views and opinions expressed are those of the authors only and do not necessarily reflect those of the European Union, Research Executive Agency or UK Research \& Innovation. Neither the European Union nor UK Research \& Innovation can be held responsible for them. P.V. acknowledges supported from the United States Army Research Office (Award No. W911NF-21-S-0009-2). The views and conclusions contained herein are those of the authors and should not be interpreted as necessarily representing the official policies or endorsements, either expressed or implied, of the Army Research Office, or the U.S. Government. This project is co-funded by the European Union and UK Research \& Innovation (Quantum Flagship project
ASPECTS, Grant Agreement No. 101080167).
\end{ack}

\medskip

{
\small

\bibliography{refs}

@book{sutton2018reinforcement,
  title     = {Reinforcement Learning: An Introduction},
  author    = {Sutton, Richard S. and Barto, Andrew G.},
  year      = {2018},
  edition   = {2nd},
  publisher = {MIT Press},
  address   = {Cambridge, MA}
}

@article{bernstein2002complexity,
  title={The complexity of decentralized control of Markov decision processes},
  author={Bernstein, Daniel S and Givan, Robert and Immerman, Neil and Zilberstein, Shlomo},
  journal={Mathematics of operations research},
  volume={27},
  number={4},
  pages={819--840},
  year={2002},
  publisher={INFORMS}
}

@article{mnih2015human,
  title={Human-level control through deep reinforcement learning},
  author={Mnih, Volodymyr and Kavukcuoglu, Koray and Silver, David and Rusu, Andrei A and Veness, Joel and Bellemare, Marc G and Graves, Alex and Riedmiller, Martin and Fidjeland, Andreas K and Ostrovski, Georg and others},
  journal={nature},
  volume={518},
  number={7540},
  pages={529--533},
  year={2015},
  publisher={Nature Publishing Group}
}

@article{baum2021experimental,
  title={Experimental deep reinforcement learning for error-robust gate-set design on a superconducting quantum computer},
  author={Baum, Yuval and Amico, Mirko and Howell, Sean and Hush, Michael and Liuzzi, Maggie and Mundada, Pranav and Merkh, Thomas and Carvalho, Andre RR and Biercuk, Michael J},
  journal={PRX quantum},
  volume={2},
  number={4},
  pages={040324},
  year={2021},
  publisher={APS}
}

@inproceedings{semola2022deep,
  title={Deep reinforcement learning quantum control on ibmq platforms and qiskit pulse},
  author={Semola, Rudy and Moro, Lorenzo and Bacciu, Davide and Prati, Enrico},
  booktitle={2022 ieee international conference on quantum computing and engineering (qce)},
  pages={759--762},
  year={2022},
  organization={IEEE}
}

@article{reuer2023realizing,
  title={Realizing a deep reinforcement learning agent for real-time quantum feedback},
  author={Reuer, Kevin and Landgraf, Jonas and F{\"o}sel, Thomas and O’Sullivan, James and Beltr{\'a}n, Liberto and Akin, Abdulkadir and Norris, Graham J and Remm, Ants and Kerschbaum, Michael and Besse, Jean-Claude and others},
  journal={Nature Communications},
  volume={14},
  number={1},
  pages={7138},
  year={2023},
  publisher={Nature Publishing Group UK London}
}

@article{nguyen2021deep,
  title={Deep reinforcement learning for efficient measurement of quantum devices},
  author={Nguyen, V and Orbell, SB and Lennon, Dominic T and Moon, Hyungil and Vigneau, Florian and Camenzind, Leon C and Yu, Liuqi and Zumb{\"u}hl, Dominik M and Briggs, G Andrew D and Osborne, Michael A and others},
  journal={npj Quantum Information},
  volume={7},
  number={1},
  pages={100},
  year={2021},
  publisher={Nature Publishing Group UK London}
}

@article{cao2026qcaleval,
  title={QCalEval: Benchmarking Vision-Language Models for Quantum Calibration Plot Understanding},
  author={Cao, Shuxiang and Zhang, Zijian and Agarwal, Abhishek and Bratrud, Grace and Beysengulov, Niyaz R and Cole, Daniel C and Frieiro, Alejandro G{\'o}mez and Glen, Elena O and Hsu, Hao and Huang, Gang and others},
  journal={arXiv preprint arXiv:2604.25884},
  year={2026}
}

@article{ares2021machine,
  title={Machine learning as an enabler of qubit scalability},
  author={Ares, Natalia},
  journal={Nature Reviews Materials},
  volume={6},
  number={10},
  pages={870--871},
  year={2021},
  publisher={Nature Publishing Group UK London}
}

@article{zwolak2024data,
  title={Data needs and challenges for quantum dot devices automation},
  author={Zwolak, Justyna P and Taylor, Jacob M and Andrews, Reed W and Benson, Jared and Bryant, Garnett W and Buterakos, Donovan and Chatterjee, Anasua and Das Sarma, Sankar and Eriksson, Mark A and Greplov{\'a}, Eli{\v{s}}ka and others},
  journal={npj Quantum Information},
  volume={10},
  number={1},
  pages={105},
  year={2024},
  publisher={Nature Publishing Group UK London}
}

@article{kondo2025environment,
  title={Environment model construction toward auto-tuning of quantum dot devices based on model-based reinforcement learning},
  author={Kondo, Chihiro and Mizokuchi, Raisei and Yoneda, Jun and Kodera, Tetsuo},
  journal={APL Machine Learning},
  volume={3},
  number={1},
  year={2025},
  publisher={AIP Publishing}
}

@article{chatterjee2021semiconductor,
  title={Semiconductor qubits in practice},
  author={Chatterjee, Anasua and Stevenson, Paul and De Franceschi, Silvano and Morello, Andrea and de Leon, Nathalie P and Kuemmeth, Ferdinand},
  journal={Nature Reviews Physics},
  volume={3},
  number={3},
  pages={157--177},
  year={2021},
  publisher={Nature Publishing Group UK London}
}

@article{bartee2025spin,
  title={Spin-qubit control with a milli-kelvin CMOS chip},
  author={Bartee, Samuel K and Gilbert, Will and Zuo, Kun and Das, Kushal and Tanttu, Tuomo and Yang, Chih Hwan and Dumoulin Stuyck, Nard and Pauka, Sebastian J and Su, Rocky Y and Lim, Wee Han and others},
  journal={Nature},
  pages={1--6},
  year={2025},
  publisher={Nature Publishing Group UK London}
}

@article{steinacker2025industry,
  title={Industry-compatible silicon spin-qubit unit cells exceeding 99\% fidelity},
  author={Steinacker, Paul and Dumoulin Stuyck, Nard and Lim, Wee Han and Tanttu, Tuomo and Feng, MengKe and Serrano, Santiago and Nickl, Andreas and Candido, Marco and Cifuentes, Jesus D and Vahapoglu, Ensar and others},
  journal={Nature},
  pages={1--7},
  year={2025},
  publisher={Nature Publishing Group UK London}
}

@article{dijkema2026simultaneous,
  title={Simultaneous operation of an 18-qubit modular array in germanium},
  author={Dijkema, Jurgen J and Zhang, Xin and Bardakas, Achilleas and Bouman, Daniel and Cuzzocrea, Alice and van Driel, David and Girardi, Davide and Stehouwer, Lucas EA and Scappucci, Giordano and Zwerver, Anne-Marije J and others},
  journal={arXiv preprint arXiv:2604.01063},
  year={2026}
}

@article{borsoi2024shared,
  title={Shared control of a 16 semiconductor quantum dot crossbar array},
  author={Borsoi, Francesco and Hendrickx, Nico W and John, Valentin and Meyer, Marcel and Motz, Sayr and Van Riggelen, Floor and Sammak, Amir and De Snoo, Sander L and Scappucci, Giordano and Veldhorst, Menno},
  journal={Nature Nanotechnology},
  volume={19},
  number={1},
  pages={21--27},
  year={2024},
  publisher={Nature Publishing Group UK London}
}

@article{volk2019loading,
  title={Loading a quantum-dot based “Qubyte” register},
  author={Volk, Christian and Zwerver, Anne-Marije J and Mukhopadhyay, Uditendu and Eendebak, Pieter T and van Diepen, Cornelis Jacobus and Dehollain, Juan Pablo and Hensgens, Toivo and Fujita, Takafumi and Reichl, Christian and Wegscheider, Werner and others},
  journal={npj Quantum Information},
  volume={5},
  number={1},
  pages={29},
  year={2019},
  publisher={Nature Publishing Group UK London}
}

@article{tsuzuki2026cnn,
  title={CNN-assisted automatic cross-capacitance matrix update for virtual-gate control of quantum dot arrays},
  author={Tsuzuki, Tatsuo and Yuta, Hideaki and Muto, Yui and Ludwig, Arne and Wieck, Andreas Dirk and Oiwa, Akira and Fujita, Takafumi},
  journal={Japanese Journal of Applied Physics},
  volume={65},
  number={1},
  pages={01SP03},
  year={2026},
  publisher={IOP Publishing}
}

@article{rao2025modular,
  title={Modular autonomous virtualization system for two-dimensional semiconductor quantum dot arrays},
  author={Rao, Anantha S and Buterakos, Donovan and van Straaten, Barnaby and John, Valentin and Yu, C{\'e}cile X and Oosterhout, Stefan D and Stehouwer, Lucas and Scappucci, Giordano and Veldhorst, Menno and Borsoi, Francesco and others},
  journal={Physical Review X},
  volume={15},
  number={2},
  pages={021034},
  year={2025},
  publisher={APS}
}

@article{muto2025automatic,
  title={Automatic detection of single-electron regime of quantum dots and definition of virtual gates using U-Net and clustering},
  author={Muto, Yui and Zielewski, Michael R and Shinozaki, Motoya and Noro, Kosuke and Otsuka, Tomohiro},
  journal={arXiv preprint arXiv:2501.05878},
  year={2025}
}

@article{hader2025automated,
  title={Automated Charge Transition Detection in Quantum Dot Charge Stability Diagrams},
  author={Hader, Fabian and Fuchs, Fabian and Fleitmann, Sarah and Havemann, Karin and Scherer, Benedikt and Vogelbruch, Jan and Geck, Lotte and Van Waasen, Stefan},
  journal={IEEE Transactions on Quantum Engineering},
  year={2025},
  publisher={IEEE}
}

@article{marchand2025end,
  title={End-to-End Analysis of Charge Stability Diagrams with Transformers},
  author={Marchand, Rahul and Schorling, Lucas and Carlsson, Cornelius and Schuff, Jonas and van Straaten, Barnaby and Patti, Taylor L and Fedele, Federico and Ziegler, Joshua and Girdhar, Parth and Vaidhyanathan, Pranav and others},
  journal={arXiv preprint arXiv:2508.15710},
  year={2025}
}

@article{oakes2020automatic,
  title={Automatic virtual voltage extraction of a 2x2 array of quantum dots with machine learning},
  author={Oakes, Giovanni A and Duan, Jingyu and Morton, John JL and Lee, Alpha and Smith, Charles G and Zalba, M},
  journal={arXiv preprint arXiv:2012.03685},
  year={2020}
}

@article{baart2016computer,
  title={Computer-automated tuning of semiconductor double quantum dots into the single-electron regime},
  author={Baart, Timothy A and Eendebak, Pieter T and Reichl, Christian and Wegscheider, Werner and Vandersypen, Lieven MK},
  journal={Applied Physics Letters},
  volume={108},
  number={21},
  year={2016},
  publisher={AIP Publishing}
}

@article{yon2025experimental,
  title={Experimental online quantum dots charge autotuning using neural networks},
  author={Yon, Victor and Galaup, Bastien and Rohrbacher, Claude and Rivard, Joffrey and Morel, Alexis and Leclerc, Dominic and Godfrin, Clement and Li, Ruoyu and Kubicek, Stefan and Greve, Kristiaan De and others},
  journal={Nano Letters},
  volume={25},
  number={10},
  pages={3717--3725},
  year={2025},
  publisher={ACS Publications}
}

@article{roux2025rapid,
  title={Rapid Autotuning of a SiGe Quantum Dot into the Single-Electron Regime with Machine Learning and RF-Reflectometry FPGA-Based Measurements},
  author={Roux, Marc-Antoine and Rivard, Joffrey and Yon, Victor and Morel, Alexis and Leclerc, Dominic and Rohrbacher, Claude and Ndiaye, El Bachir and Tafuri, Felice Francesco and Bono, Brendan and Kubicek, Stefan and others},
  journal={arXiv preprint arXiv:2509.19537},
  year={2025}
}

@article{kalantre2019machine,
  title={Machine learning techniques for state recognition and auto-tuning in quantum dots},
  author={Kalantre, Sandesh S and Zwolak, Justyna P and Ragole, Stephen and Wu, Xingyao and Zimmerman, Neil M and Stewart Jr, MD and Taylor, Jacob M},
  journal={npj Quantum Information},
  volume={5},
  number={1},
  pages={6},
  year={2019},
  publisher={Nature Publishing Group UK London}
}

@article{zwolak2020autotuning,
  title={Autotuning of double-dot devices in situ with machine learning},
  author={Zwolak, Justyna P and McJunkin, Thomas and Kalantre, Sandesh S and Dodson, JP and MacQuarrie, ER and Savage, DE and Lagally, MG and Coppersmith, SN and Eriksson, Mark A and Taylor, Jacob M},
  journal={Physical review applied},
  volume={13},
  number={3},
  pages={034075},
  year={2020},
  publisher={APS}
}

@article{ziegler2023tuning,
  title={Tuning arrays with rays: Physics-informed tuning of quantum dot charge states},
  author={Ziegler, Joshua and Luthi, Florian and Ramsey, Mick and Borjans, Felix and Zheng, Guoji and Zwolak, Justyna P},
  journal={Physical Review Applied},
  volume={20},
  number={3},
  pages={034067},
  year={2023},
  publisher={APS}
}

@article{czischek2021miniaturizing,
  title={Miniaturizing neural networks for charge state autotuning in quantum dots},
  author={Czischek, Stefanie and Yon, Victor and Genest, Marc-Antoine and Roux, Marc-Antoine and Rochette, Sophie and Lemyre, Julien Camirand and Moras, Mathieu and Pioro-Ladri{\`e}re, Michel and Drouin, Dominique and Beilliard, Yann and others},
  journal={Machine Learning: Science and Technology},
  volume={3},
  number={1},
  pages={015001},
  year={2021},
  publisher={IOP Publishing}
}

@article{durrer2020automated,
  title={Automated tuning of double quantum dots into specific charge states using neural networks},
  author={Durrer, Renato and Kratochwil, Benedikt and Koski, Jonne V and Landig, Andreas J and Reichl, Christian and Wegscheider, Werner and Ihn, Thomas and Greplova, Eliska},
  journal={Physical Review Applied},
  volume={13},
  number={5},
  pages={054019},
  year={2020},
  publisher={APS}
}

@article{moon2020machine,
  title={Machine learning enables completely automatic tuning of a quantum device faster than human experts},
  author={Moon, Hyungil and Lennon, Dominic T and Kirkpatrick, James and van Esbroeck, Nina M and Camenzind, Leon C and Yu, Liuqi and Vigneau, Florian and Zumb{\"u}hl, Dominik M and Briggs, G Andrew D and Osborne, Michael A and others},
  journal={Nature communications},
  volume={11},
  number={1},
  pages={4161},
  year={2020},
  publisher={Nature Publishing Group UK London}
}

@article{severin2024cross,
  title={Cross-architecture tuning of silicon and SiGe-based quantum devices using machine learning},
  author={Severin, Brandon and Lennon, Dominic T and Camenzind, Leon C and Vigneau, Florian and Fedele, Federico and Jirovec, Daniel and Ballabio, Andrea and Chrastina, Daniel and Isella, Giovanni and de Kruijf, Mathieu and others},
  journal={Scientific Reports},
  volume={14},
  number={1},
  pages={17281},
  year={2024},
  publisher={Nature Publishing Group UK London}
}

@article{schuff2026fully,
  title={Fully autonomous tuning of a spin qubit},
  author={Schuff, Jonas and Carballido, Miguel J and Kotzagiannidis, Madeleine and Calvo, Juan Carlos and Caselli, Marco and Rawling, Jacob and Craig, David L and van Straaten, Barnaby and Severin, Brandon and Fedele, Federico and others},
  journal={Nature Electronics},
  pages={1--10},
  year={2026},
  publisher={Nature Publishing Group UK London}
}

@article{van2025all,
  title={All-rf-based coarse-tuning algorithm for quantum devices using machine learning},
  author={Van Straaten, Barnaby and Fedele, Federico and Vigneau, Florian and Hickie, Joseph and Jirovec, Daniel and Ballabio, Andrea and Chrastina, Daniel and Isella, Giovanni and Katsaros, Georgios and Ares, Natalia},
  journal={Physical Review Applied},
  volume={24},
  number={5},
  pages={054030},
  year={2025},
  publisher={APS}
}

@article{van2020quantum,
  title={Quantum device fine-tuning using unsupervised embedding learning},
  author={van Esbroeck, Nina M and Lennon, Dominic T and Moon, Hyungil and Nguyen, Vu and Vigneau, Florian and Camenzind, Leon C and Yu, Liuqi and Zumb{\"u}hl, Dominik M and Briggs, G Andrew D and Sejdinovic, Dino and others},
  journal={New Journal of Physics},
  volume={22},
  number={9},
  pages={095003},
  year={2020},
  publisher={IOP Publishing}
}

@article{van2024qarray,
  title={QArray: A GPU-accelerated constant capacitance model simulator for large quantum dot arrays},
  author={van Straaten, Barnaby and Hickie, Joseph and Schorling, Lucas and Schuff, Jonas and Fedele, Federico and Ares, Natalia},
  journal={SciPost Physics Codebases},
  pages={035},
  year={2024}
}

@article{van2024codebase,
  title={Codebase release 1.3 for QArray},
  author={van Straaten, Barnaby and Hickie, Joseph and Schorling, Lucas and Schuff, Jonas and Fedele, Federico and Ares, Natalia},
  journal={SciPost Physics Codebases},
  pages={035},
  year={2024}
}

@inproceedings{christianos2021scaling,
  title={Scaling multi-agent reinforcement learning with selective parameter sharing},
  author={Christianos, Filippos and Papoudakis, Georgios and Rahman, Muhammad A and Albrecht, Stefano V},
  booktitle={International Conference on Machine Learning},
  pages={1989--1998},
  year={2021},
  organization={PMLR}
}

@article{ppo,
  title={Proximal policy optimization algorithms},
  author={Schulman, John and Wolski, Filip and Dhariwal, Prafulla and Radford, Alec and Klimov, Oleg},
  journal={arXiv preprint arXiv:1707.06347},
  year={2017}
}

@article{neumann2022scaling,
  title={Scaling laws for a multi-agent reinforcement learning model},
  author={Neumann, Oren and Gros, Claudius},
  journal={arXiv preprint arXiv:2210.00849},
  year={2022}
}

@article{alexeev2025artificial,
  title={Artificial intelligence for quantum computing},
  author={Alexeev, Yuri and Farag, Marwa H and Patti, Taylor L and Wolf, Mark E and Ares, Natalia and Aspuru-Guzik, Al{\'a}n and Benjamin, Simon C and Cai, Zhenyu and Cao, Shuxiang and Chamberland, Christopher and others},
  journal={Nature Communications},
  volume={16},
  number={1},
  pages={10829},
  year={2025},
  publisher={Nature Publishing Group UK London}
}

@article{carlsson2025automated,
  title={Automated All-RF Tuning for Spin Qubit Readout and Control},
  author={Carlsson, Cornelius and Saez-Mollejo, Jaime and Fedele, Federico and Calcaterra, Stefano and Chrastina, Daniel and Isella, Giovanni and Katsaros, Georgios and Ares, Natalia},
  journal={arXiv preprint arXiv:2506.10834},
  year={2025}
}

@article{hu2020learning,
  title={Learning to utilize shaping rewards: A new approach of reward shaping},
  author={Hu, Yujing and Wang, Weixun and Jia, Hangtian and Wang, Yixiang and Chen, Yingfeng and Hao, Jianye and Wu, Feng and Fan, Changjie},
  journal={Advances in Neural Information Processing Systems},
  volume={33},
  pages={15931--15941},
  year={2020}
}

@article{vaswani2017attention,
  title={Attention is all you need},
  author={Vaswani, Ashish and Shazeer, Noam and Parmar, Niki and Uszkoreit, Jakob and Jones, Llion and Gomez, Aidan N and Kaiser, {\L}ukasz and Polosukhin, Illia},
  journal={Advances in neural information processing systems},
  volume={30},
  year={2017}
}

@inproceedings{gupta2017cooperative,
  title={Cooperative multi-agent control using deep reinforcement learning},
  author={Gupta, Jayesh K and Egorov, Maxim and Kochenderfer, Mykel},
  booktitle={International conference on autonomous agents and multiagent systems},
  pages={66--83},
  year={2017},
  organization={Springer}
}

@article{oroojlooy2023review,
  title={A review of cooperative multi-agent deep reinforcement learning},
  author={Oroojlooy, Afshin and Hajinezhad, Davood},
  journal={Applied Intelligence},
  volume={53},
  number={11},
  pages={13677--13722},
  year={2023},
  publisher={Springer}
}

@inproceedings{tan1993multi,
  title={Multi-agent reinforcement learning: Independent vs. cooperative agents},
  author={Tan, Ming and others},
  booktitle={Proceedings of the tenth international conference on machine learning},
  pages={330--337},
  year={1993}
}

@article{sunehag2017value,
  title={Value-decomposition networks for cooperative multi-agent learning},
  author={Sunehag, Peter and Lever, Guy and Gruslys, Audrunas and Czarnecki, Wojciech Marian and Zambaldi, Vinicius and Jaderberg, Max and Lanctot, Marc and Sonnerat, Nicolas and Leibo, Joel Z and Tuyls, Karl and others},
  journal={arXiv preprint arXiv:1706.05296},
  year={2017}
}

@article{rashid2020monotonic,
  title={Monotonic value function factorisation for deep multi-agent reinforcement learning},
  author={Rashid, Tabish and Samvelyan, Mikayel and De Witt, Christian Schroeder and Farquhar, Gregory and Foerster, Jakob and Whiteson, Shimon},
  journal={Journal of Machine Learning Research},
  volume={21},
  number={178},
  pages={1--51},
  year={2020}
}

@inproceedings{son2019qtran,
  title={Qtran: Learning to factorize with transformation for cooperative multi-agent reinforcement learning},
  author={Son, Kyunghwan and Kim, Daewoo and Kang, Wan Ju and Hostallero, David Earl and Yi, Yung},
  booktitle={International conference on machine learning},
  pages={5887--5896},
  year={2019},
  organization={PMLR}
}

@article{de2020deep,
  title={Deep multi-agent reinforcement learning for decentralized continuous cooperative control},
  author={de Witt, Christian Schroeder and Peng, Bei and Kamienny, Pierre-Alexandre and Torr, Philip and B{\"o}hmer, Wendelin and Whiteson, Shimon},
  journal={arXiv preprint arXiv:2003.06709},
  volume={19},
  year={2020}
}

@inproceedings{foerster2018counterfactual,
  title={Counterfactual multi-agent policy gradients},
  author={Foerster, Jakob and Farquhar, Gregory and Afouras, Triantafyllos and Nardelli, Nantas and Whiteson, Shimon},
  booktitle={Proceedings of the AAAI conference on artificial intelligence},
  volume={32},
  number={1},
  year={2018}
}

@article{lowe2017multi,
  title={Multi-agent actor-critic for mixed cooperative-competitive environments},
  author={Lowe, Ryan and Wu, Yi I and Tamar, Aviv and Harb, Jean and Pieter Abbeel, OpenAI and Mordatch, Igor},
  journal={Advances in neural information processing systems},
  volume={30},
  year={2017}
}

@article{peng2021facmac,
  title={Facmac: Factored multi-agent centralised policy gradients},
  author={Peng, Bei and Rashid, Tabish and Schroeder de Witt, Christian and Kamienny, Pierre-Alexandre and Torr, Philip and B{\"o}hmer, Wendelin and Whiteson, Shimon},
  journal={Advances in neural information processing systems},
  volume={34},
  pages={12208--12221},
  year={2021}
}

@article{sukhbaatar2016learning,
  title={Learning multiagent communication with backpropagation},
  author={Sukhbaatar, Sainbayar and Fergus, Rob and others},
  journal={Advances in neural information processing systems},
  volume={29},
  year={2016}
}

@inproceedings{khan2020graph,
  title={Graph policy gradients for large scale robot control},
  author={Khan, Arbaaz and Tolstaya, Ekaterina and Ribeiro, Alejandro and Kumar, Vijay},
  booktitle={Conference on robot learning},
  pages={823--834},
  year={2020},
  organization={PMLR}
}

@inproceedings{nayak2023scalable,
  title={Scalable multi-agent reinforcement learning through intelligent information aggregation},
  author={Nayak, Siddharth and Choi, Kenneth and Ding, Wenqi and Dolan, Sydney and Gopalakrishnan, Karthik and Balakrishnan, Hamsa},
  booktitle={International conference on machine learning},
  pages={25817--25833},
  year={2023},
  organization={PMLR}
}

@article{de2020independent,
  title={Is independent learning all you need in the starcraft multi-agent challenge?},
  author={De Witt, Christian Schroeder and Gupta, Tarun and Makoviichuk, Denys and Makoviychuk, Viktor and Torr, Philip HS and Sun, Mingfei and Whiteson, Shimon},
  journal={arXiv preprint arXiv:2011.09533},
  year={2020}
}

@article{yu2022surprising,
  title={The surprising effectiveness of ppo in cooperative multi-agent games},
  author={Yu, Chao and Velu, Akash and Vinitsky, Eugene and Gao, Jiaxuan and Wang, Yu and Bayen, Alexandre and Wu, Yi},
  journal={Advances in neural information processing systems},
  volume={35},
  pages={24611--24624},
  year={2022}
}

@article{kuba2021trust,
  title={Trust region policy optimisation in multi-agent reinforcement learning},
  author={Kuba, Jakub Grudzien and Chen, Ruiqing and Wen, Muning and Wen, Ying and Sun, Fanglei and Wang, Jun and Yang, Yaodong},
  journal={arXiv preprint arXiv:2109.11251},
  year={2021}
}

@article{zhong2024heterogeneous,
  title={Heterogeneous-agent reinforcement learning},
  author={Zhong, Yifan and Kuba, Jakub Grudzien and Feng, Xidong and Hu, Siyi and Ji, Jiaming and Yang, Yaodong},
  journal={Journal of Machine Learning Research},
  volume={25},
  number={32},
  pages={1--67},
  year={2024}
}

@article{tessera2024hypermarl,
  title={Hypermarl: Adaptive hypernetworks for multi-agent rl},
  author={Tessera, Kale-ab Abebe and Rahman, Arrasy and Storkey, Amos and Albrecht, Stefano V},
  journal={arXiv preprint arXiv:2412.04233},
  year={2024}
}

@article{kilinc2018multi,
  title={Multi-agent deep reinforcement learning with extremely noisy observations},
  author={Kilinc, Ozsel and Montana, Giovanni},
  journal={arXiv preprint arXiv:1812.00922},
  year={2018}
}

@inproceedings{feng2024hierarchical,
  title={Hierarchical consensus-based multi-agent reinforcement learning for multi-robot cooperation tasks},
  author={Feng, Pu and Liang, Junkang and Wang, Size and Yu, Xin and Ji, Xin and Chen, Yiting and Zhang, Kui and Shi, Rongye and Wu, Wenjun},
  booktitle={2024 IEEE/RSJ International Conference on Intelligent Robots and Systems (IROS)},
  pages={642--649},
  year={2024},
  organization={IEEE}
}

@article{zhao2026bridging,
  title={Bridging MARL to SARL: An Order-Independent Multi-Agent Transformer via Latent Consensus},
  author={Zhao, Zijian and Gao, Jing and Li, Sen},
  journal={arXiv preprint arXiv:2604.13472},
  year={2026}
}

@article{wen2022multi,
  title={Multi-agent reinforcement learning is a sequence modeling problem},
  author={Wen, Muning and Kuba, Jakub and Lin, Runji and Zhang, Weinan and Wen, Ying and Wang, Jun and Yang, Yaodong},
  journal={Advances in Neural Information Processing Systems},
  volume={35},
  pages={16509--16521},
  year={2022}
}

@misc{dreamer,
      title={Mastering Diverse Domains through World Models}, 
      author={Danijar Hafner and Jurgis Pasukonis and Jimmy Ba and Timothy Lillicrap},
      year={2024},
      eprint={2301.04104},
      archivePrefix={arXiv},
      primaryClass={cs.AI},
      url={https://arxiv.org/abs/2301.04104}, 
}

@article{chen2003bayesian,
  title={Bayesian filtering: From Kalman filters to particle filters, and beyond},
  author={Chen, Zhe and others},
  journal={Statistics},
  volume={182},
  number={1},
  pages={1--69},
  year={2003}
}

@article{fang2018nonlinear,
  title={Nonlinear Bayesian estimation: From Kalman filtering to a broader horizon},
  author={Fang, Huazhen and Tian, Ning and Wang, Yebin and Zhou, MengChu and Haile, Mulugeta A},
  journal={IEEE/CAA Journal of Automatica Sinica},
  volume={5},
  number={2},
  pages={401--417},
  year={2018},
  publisher={IEEE}
}

@article{konda1999actor,
  title={Actor-critic algorithms},
  author={Konda, Vijay and Tsitsiklis, John},
  journal={Advances in neural information processing systems},
  volume={12},
  year={1999}
}

@article{metasym,
  title={MetaSym: A Symplectic Meta-learning Framework for Physical Intelligence},
  author={Vaidhyanathan, Pranav and Papatheodorou, Aristotelis and Mitchison, Mark T and Ares, Natalia and Havoutis, Ioannis},
  journal={arXiv preprint arXiv:2502.16667},
  year={2025}
}

@article{schorling2025meta,
  title={Meta-learning characteristics and dynamics of quantum systems},
  author={Schorling, Lucas and Vaidhyanathan, Pranav and Schuff, Jonas and Carballido, Miguel J and Zumb{\"u}hl, Dominik and Milburn, Gerard and Marquardt, Florian and Foerster, Jakob and Osborne, Michael A and Ares, Natalia},
  journal={arXiv preprint arXiv:2503.10492},
  year={2025}
}

@article{bukov2026reinforcement,
  title={Reinforcement Learning for Quantum Technology},
  author={Bukov, Marin and Marquardt, Florian},
  journal={arXiv preprint arXiv:2601.18953},
  year={2026}
}

@article{vaidhyanathan2026quantum,
  title={Quantum feedback control with a transformer neural network architecture},
  author={Vaidhyanathan, Pranav and Marquardt, Florian and Mitchison, Mark T and Ares, Natalia},
  journal={Physical Review Research},
  volume={8},
  number={1},
  pages={L012043},
  year={2026},
  publisher={APS}
}

@misc{lillicrap2020continuous,
  title={Continuous control with deep reinforcement learning},
  author={Lillicrap, Timothy Paul and Hunt, Jonathan James and Pritzel, Alexander and Heess, Nicolas Manfred Otto and Erez, Tom and Tassa, Yuval and Silver, David and Wierstra, Daniel Pieter},
  year={2020},
  month=sep # "~15",
  publisher={Google Patents},
  note={US Patent 10,776,692}
}

@inproceedings{bandit, 
    title={Survey on applications of multi-armed and contextual bandits}, 
    author={Bouneffouf, Djallel and Rish, Irina and Aggarwal, Charu}, 
    booktitle={2020 IEEE congress on evolutionary computation (CEC)}, 
    pages={1--8}, 
    year={2020}, 
    organization={IEEE} 
}

@inproceedings{xu2018experience,
  title={Experience-driven networking: A deep reinforcement learning based approach},
  author={Xu, Zhiyuan and Tang, Jian and Meng, Jingsong and Zhang, Weiyi and Wang, Yanzhi and Liu, Chi Harold and Yang, Dejun},
  booktitle={IEEE INFOCOM 2018-IEEE conference on computer communications},
  pages={1871--1879},
  year={2018},
  organization={IEEE}
}

@article{hochreiter1997lstm,
  title={Long short-term memory},
  author={Hochreiter, Sepp and Schmidhuber, J{"u}rgen},
  journal={Neural Computation},
  volume={9},
  number={8},
  pages={1735--1780},
  year={1997},
  publisher={MIT Press}
}

@inproceedings{he2016resnet,
  title={Deep residual learning for image recognition},
  author={He, Kaiming and Zhang, Xiangyu and Ren, Shaoqing and Sun, Jian},
  booktitle={Proceedings of the IEEE Conference on Computer Vision and Pattern Recognition (CVPR)},
  pages={770--778},
  year={2016}
}

@article{nelder1965simplex,
  title={A simplex method for function minimization},
  author={Nelder, John A and Mead, Roger},
  journal={The Computer Journal},
  volume={7},
  number={4},
  pages={308--313},
  year={1965},
  publisher={Oxford University Press}
}

@book{rasmussen2006gpml,
  title={Gaussian Processes for Machine Learning},
  author={Rasmussen, Carl Edward and Williams, Christopher K I},
  year={2006},
  publisher={MIT Press}
}

@article{shahriari2016taking,
  title={Taking the human out of the loop: A review of Bayesian optimization},
  author={Shahriari, Bobak and Swersky, Kevin and Wang, Ziyu and Adams, Ryan P and de Freitas, Nando},
  journal={Proceedings of the IEEE},
  volume={104},
  number={1},
  pages={148--175},
  year={2016},
  publisher={IEEE}
}

@article{nocedal1980lbfgs,
  title={Updating quasi-Newton matrices with limited storage},
  author={Nocedal, Jorge},
  journal={Mathematics of Computation},
  volume={35},
  number={151},
  pages={773--782},
  year={1980}
}

@book{massart2007concentration,
  title={Concentration inequalities and model selection: Ecole d'Et{\'e} de Probabilit{\'e}s de Saint-Flour XXXIII-2003},
  author={Massart, Pascal},
  year={2007},
  publisher={Springer}
}

@article{li2024kaleidoscope,
  title={Kaleidoscope: Learnable masks for heterogeneous multi-agent reinforcement learning},
  author={Li, Xinran and Pan, Ling and Zhang, Jun},
  journal={Advances in Neural Information Processing Systems},
  volume={37},
  pages={22081--22106},
  year={2024}
}

@unpublished{guilmin2025dynamiqs,
  title  = {Dynamiqs: an open-source Python library for GPU-accelerated and differentiable simulation of quantum systems},
  author = {Pierre Guilmin and Adrien Bocquet and {\'{E}}lie Genois and Daniel Weiss and Ronan Gautier},
  year   = {2025},
  url    = {https://github.com/dynamiqs/dynamiqs}
}

@article{koch2007charge,
  title={Charge-insensitive qubit design derived from the Cooper pair box},
  author={Koch, Jens and Yu, Terri M and Gambetta, Jay and Houck, Andrew A and Schuster, David I and Majer, Johannes and Blais, Alexandre and Devoret, Michel H and Girvin, Steven M and Schoelkopf, Robert J},
  journal={Physical Review A—Atomic, Molecular, and Optical Physics},
  volume={76},
  number={4},
  pages={042319},
  year={2007},
  publisher={APS}
}

@article{bertrand2015quantum,
  title={Quantum manipulation of two-electron spin states in isolated double quantum dots},
  author={Bertrand, Benoit and Flentje, Hanno and Takada, Shintaro and Yamamoto, Michihisa and Tarucha, Seigo and Ludwig, Arne and Wieck, Andreas D and B{\"a}uerle, Christopher and Meunier, Tristan},
  journal={Physical review letters},
  volume={115},
  number={9},
  pages={096801},
  year={2015},
  publisher={APS}
}

@article{wesdorp2026mitigating,
  title={Mitigating crosstalk errors for simultaneous single-qubit gates on a superconducting quantum processor},
  author={Wesdorp, Jaap J and Hyypp{\"a}, Eric and Andersson, Joona and Adam, Janos and Beriwal, Rohit and Bergholm, Ville and Dahl, Saga and Fasciati, Simone Diego and Friero, Alejandro Gomez and Gao, Zheming and others},
  journal={arXiv preprint arXiv:2603.11018},
  year={2026}
}

@article{gao2021practical,
  title={Practical guide for building superconducting quantum devices},
  author={Gao, Yvonne Y and Rol, M Adriaan and Touzard, Steven and Wang, Chen},
  journal={PRX quantum},
  volume={2},
  number={4},
  pages={040202},
  year={2021},
  publisher={APS}
}

@book{reed2013entanglement,
  title={Entanglement and quantum error correction with superconducting qubits},
  author={Reed, Matthew},
  year={2013},
  publisher={Lulu. com}
}
\bibliographystyle{unsrt}

}


\newpage 

\appendix

\section{Proofs on Virtualization and Filters}
\label{app:proof}

\paragraph{Notation.}

For $z\in\mathbb{R}^d$, define $\|z\|_2 := \sqrt{z^\top z}$.
For a matrix $A\in\mathbb{R}^{m\times n}$, define the spectral norm
$$
\|A\|_2 := \sup_{\|x\|_2 = 1} \|Ax\|_2.
$$
Recall that \(\|A\|_2\) equals the largest singular value $\sigma_{\max}(A)$, and for any $x$,
$$
\|Ax\|_2 \le \|A\|_2 \, \|x\|_2
\qquad\text{(by the definition of the supremum)}.
$$
We also use the elementary implication: if $0 \prec B \preceq A$ (PSD order), then $A^{-1} \preceq B^{-1}$.

\subsubsection{Local linearization model}
We start by fixing an operating point $u_0\in\mathbb{R}^D$ (plunger voltages). Let $x(u)\in\mathbb{R}^D$
denote the local device state. Let $J(u) := \nabla x(u)\in\mathbb{R}^{D\times D}$.

\begin{assumption}[Locally Lipschitz Jacobian]
\label{ass:lipschitz_jacobian}
There exists a convex neighborhood $mathcal{U}\subset\mathbb{R}^D$ and a constant $L\ge 0$ such that
for all $u,v\in\mathcal{U}$,
$$
\|J(u)-J(v)\|_2 \le L\|u-v\|_2.
$$
\end{assumption}

\begin{lemma}[First-order expansion with explicit remainder bound]
\label{lem:taylor_remainder}
Fix \(u_0\in\mathcal{U}\) and define \(C := J(u_0)\). For any \(u\in\mathcal{U}\),
there exists a remainder vector \(r(u)\in\mathbb{R}^D\) such that
$$
x(u) = x(u_0) + C(u-u_0) + r(u),
$$
and \(r(u)\) satisfies the explicit bound
$$
\|r(u)\|_2 \le \frac{L}{2}\|u-u_0\|_2^2.
$$
\end{lemma}

\begin{proof}
Define the path \(\gamma:[0,1]\to\mathcal{U}\) by \(\gamma(t)=u_0+t(u-u_0)\) (convexity ensures \(\gamma(t)\in\mathcal{U}\)).
By the fundamental theorem of calculus,
$$
x(u)-x(u_0) = \int_0^1 \frac{d}{dt}x(\gamma(t))\,dt
           = \int_0^1 J(\gamma(t))(u-u_0)\,dt.
$$
Add and subtract \(C(u-u_0)\) inside the integral and define
$$
r(u) := \int_0^1 \big(J(\gamma(t)) - C\big)(u-u_0)\,dt.
$$
Then the expansion holds. For the bound, use the definition of \(\|\cdot\|_2\) and Assumption~\ref{ass:lipschitz_jacobian}:
for each \(t\in[0,1]\),
$$
\| \big(J(\gamma(t)) - C\big)(u-u_0)\|_2
\le \|J(\gamma(t)) - C\|_2 \,\|u-u_0\|_2
\le L\|\gamma(t)-u_0\|_2\,\|u-u_0\|_2
= Lt\|u-u_0\|_2^2.
$$
Integrating \(t\) from \(0\) to \(1\) gives
$$
\|r(u)\|_2 \le \int_0^1 Lt\|u-u_0\|_2^2\,dt = \frac{L}{2}\|u-u_0\|_2^2.
$$
\end{proof}

\subsubsection{One-step quadratic objective and greedy optimum}
Let \(x^\star\in\mathbb{R}^D\) be a target and define \(e(u):=x(u)-x^\star\).
At iteration \(t\), the algorithm chooses \(\Delta u_t\) and applies \(u_{t+1}=u_t+\Delta u_t\).
In the linearized model (dropping \(r(\cdot)\) for the moment), the one-step surrogate cost is
$$
J_t(\Delta u) := \frac12\|e_t + C\Delta u\|_2^2,
\qquad e_t := e(u_t).
$$

\begin{proposition}[Greedy optimum in raw coordinates]
\label{prop:raw_opt}
If \(C\) is invertible, the unique minimizer of \(J_t(\Delta u)\) is
$$
\Delta u_t^\star = -C^{-1}e_t,
$$
and under the linear model the next error satisfies \(e_{t+1}=0\).
\end{proposition}

\begin{proof}
Write \(J_t(\Delta u)=\tfrac12(e_t+C\Delta u)^\top(e_t+C\Delta u)\).
Differentiating w.r.t.\ \(\Delta u\) yields
$$
\nabla_{\Delta u}J_t(\Delta u) = C^\top(e_t+C\Delta u).
$$
At a minimizer, \(\nabla_{\Delta u}J_t(\Delta u)=0\), i.e.
$$
C^\top C\,\Delta u = -C^\top e_t.
$$
Since \(C\) is invertible, \(C^\top C\) is positive definite and invertible, hence the solution is unique:
$$
\Delta u = -(C^\top C)^{-1}C^\top e_t.
$$
Because \(C\) is invertible, one checks directly that \((C^\top C)^{-1}C^\top = C^{-1}\) (multiply by \(C\) on the right to obtain the identity),
so \(\Delta u_t^\star=-C^{-1}e_t\). Substituting gives \(e_{t+1}=e_t+C\Delta u_t^\star=0\).
\end{proof}

\subsubsection{Virtualization as a preconditioned coordinate system}
QADAPT constructs an estimate \(\widehat C_t\) of the cross-capacitance and uses it to define a
virtualized coordinate \(\Delta v_t\) through the linear change of variables
$$
\Delta u_t = \widehat C_t^{-1}\Delta v_t.
$$
Under the linear model, the error evolves as
$$
e_{t+1} = e_t + C\widehat C_t^{-1}\Delta v_t.
$$

\begin{theorem}[Exact cancellation in ideal virtual coordinates]
\label{thm:ideal_cancel}
Assume \(C\) and \(\widehat C_t\) are invertible and choose \(\Delta v_t=-e_t\).
Define the mismatch matrix
$$
E_t := I - C\widehat C_t^{-1}.
$$
Then
$$
e_{t+1} = E_t e_t.
$$
Moreover, if there exists \(\rho\in[0,1)\) such that \(\|E_t\|_2\le \rho\) for all \(t\), then
$$
\|e_t\|_2 \le \rho^t \|e_0\|_2 \qquad \text{for all } t\ge 0.
$$
\end{theorem}

\begin{proof}
Substitute \(\Delta v_t=-e_t\) into the update:
$$
e_{t+1} = e_t - C\widehat C_t^{-1}e_t = (I-C\widehat C_t^{-1})e_t = E_t e_t.
$$
To bound \(\|e_{t+1}\|_2\), use the definition of \(\|E_t\|_2\):
$$
\|e_{t+1}\|_2 = \|E_t e_t\|_2
\le \|E_t\|_2 \|e_t\|_2
\le \rho \|e_t\|_2.
$$
Iterating this inequality yields \(\|e_t\|_2 \le \rho^t\|e_0\|_2\).
\end{proof}

\paragraph{Mismatch in terms of estimation error.}
Let \(\Delta_t := \widehat C_t - C\).

\begin{lemma}[Mismatch identity and explicit norm bound]
\label{lem:Et_delta}
If \(\widehat C_t\) is invertible, then
$$
E_t = \Delta_t \widehat C_t^{-1}.
$$
Consequently,
$$
\|E_t\|_2 \le \|\Delta_t\|_2 \,\|\widehat C_t^{-1}\|_2.
$$
\end{lemma}

\begin{proof}
Compute
$$
C\widehat C_t^{-1} = (\widehat C_t-\Delta_t)\widehat C_t^{-1}
= \widehat C_t\widehat C_t^{-1} - \Delta_t\widehat C_t^{-1}
= I - \Delta_t\widehat C_t^{-1}.
$$
Hence \(E_t = I - C\widehat C_t^{-1}=\Delta_t\widehat C_t^{-1}\).

For the norm bound, start from the definition of \(\|E_t\|_2\):
$$
\|E_t\|_2
= \sup_{\|x\|_2=1}\|E_t x\|_2
= \sup_{\|x\|_2=1}\|\Delta_t\widehat C_t^{-1}x\|_2.
$$
For each such \(x\), let \(y:=\widehat C_t^{-1}x\). Then \(\|\Delta_t y\|_2\le \|\Delta_t\|_2\|y\|_2\).
Therefore
$$
\|\Delta_t\widehat C_t^{-1}x\|_2 \le \|\Delta_t\|_2\,\|\widehat C_t^{-1}x\|_2
\le \|\Delta_t\|_2\,\|\widehat C_t^{-1}\|_2\,\|x\|_2
= \|\Delta_t\|_2\,\|\widehat C_t^{-1}\|_2.
$$
Taking the supremum over \(\|x\|_2=1\) yields the claimed inequality.
\end{proof}

\paragraph{Invertibility and stability of the inverse.}

\begin{lemma}[Neumann-series inverse bound]
\label{lem:inv_perturb}
Let \(C\) be invertible and \(\widehat C = C+\Delta\). If \(\|C^{-1}\|_2\|\Delta\|_2<1\), then \(\widehat C\) is invertible and
$$
\|\widehat C^{-1}\|_2 \le \frac{\|C^{-1}\|_2}{1-\|C^{-1}\|_2\|\Delta\|_2}.
$$
\end{lemma}

\begin{proof}
Factor \(\widehat C\) as
$$
\widehat C = C(I+C^{-1}\Delta).
$$
Let \(X:=C^{-1}\Delta\). The assumption implies \(\|X\|_2<1\).
Consider the series \(S:=\sum_{k=0}^\infty (-X)^k\). For each partial sum \(S_n=\sum_{k=0}^n(-X)^k\),
$$
(I+X)S_n = \sum_{k=0}^n(-X)^k + \sum_{k=0}^n X(-X)^k
= I + \sum_{k=1}^n(-X)^k + \sum_{k=1}^{n+1}(-X)^k
= I - (-X)^{n+1}.
$$
Since \(\|X\|_2<1\), we have \(\|(-X)^{n+1}\|_2 \le \|X\|_2^{n+1}\to 0\), hence \(S_n\to S\) and \((I+X)S=I\).
Thus \((I+X)^{-1}=S\) exists and
$$
\|(I+X)^{-1}\|_2 = \|S\|_2
\le \sum_{k=0}^\infty \|X\|_2^k
= \frac{1}{1-\|X\|_2}.
$$
Therefore
$$
\widehat C^{-1} = (I+C^{-1}\Delta)^{-1}C^{-1},
\qquad
\|\widehat C^{-1}\|_2 \le \|(I+X)^{-1}\|_2\,\|C^{-1}\|_2
\le \frac{\|C^{-1}\|_2}{1-\|C^{-1}\|_2\|\Delta\|_2}.
$$
\end{proof}

\begin{corollary}[Explicit sufficient condition for contraction]
\label{cor:contract_condition}
Assume \(C\) is invertible and \(\widehat C_t=C+\Delta_t\).
If \(\|C^{-1}\|_2\|\Delta_t\|_2<\tfrac12\), then \(\|E_t\|_2<1\). In fact,
$$
\|E_t\|_2 \le \frac{\|C^{-1}\|_2\|\Delta_t\|_2}{1-\|C^{-1}\|_2\|\Delta_t\|_2}.
$$
\end{corollary}

\begin{proof}
By Lemma~\ref{lem:Et_delta},
$$
\|E_t\|_2 \le \|\Delta_t\|_2 \,\|\widehat C_t^{-1}\|_2.
$$
Apply Lemma~\ref{lem:inv_perturb} with \(\Delta=\Delta_t\):
$$
\|\widehat C_t^{-1}\|_2 \le \frac{\|C^{-1}\|_2}{1-\|C^{-1}\|_2\|\Delta_t\|_2}.
$$
Multiplying yields the claimed bound. Let \(a:=\|C^{-1}\|_2\|\Delta_t\|_2\). If \(a<1/2\), then
$$
\frac{a}{1-a} < \frac{1/2}{1-1/2}=1,
$$
hence \(\|E_t\|_2<1\).
\end{proof}

\subsection{Virtualization Improves Conditioning (Preconditioning)}
\label{sec:virt_preconditioning}

Consider minimizing the one-step quadratic surrogate
$$
J_t(\Delta u) := \frac12\|e_t + C\,\Delta u\|_2^2.
$$
In  voltage coordinates, the Hessian is
$$
\nabla^2_{\Delta u} J_t(\Delta u) = C^\top C,
$$
so the conditioning of the local quadratic model is governed by
$$
\kappa(C^\top C) = \frac{\lambda_{\max}(C^\top C)}{\lambda_{\min}(C^\top C)} = \kappa(C)^2,
$$
which can be large when gate cross-talk makes \(C\) ill-conditioned.

Under virtualization, we change variables via
$$
\Delta u = \widehat C_t^{-1}\Delta v.
$$
Substituting into \(J_t\) yields the virtual-coordinate objective
$$
J_t(\widehat C_t^{-1}\Delta v)
= \frac12\Big\|e_t + C\,\widehat C_t^{-1}\Delta v\Big\|_2^2
= \frac12\Big\|e_t + M_t \Delta v\Big\|_2^2,
\qquad
M_t := C\,\widehat C_t^{-1}.
$$
The Hessian with respect to \(\Delta v\) is therefore
$$
\nabla^2_{\Delta v} J_t(\widehat C_t^{-1}\Delta v) = M_t^\top M_t.
$$
Hence the conditioning of the virtual-coordinate quadratic problem is
$$
\kappa(M_t^\top M_t)
= \frac{\lambda_{\max}(M_t^\top M_t)}{\lambda_{\min}(M_t^\top M_t)}
= \frac{\sigma_{\max}(M_t)^2}{\sigma_{\min}(M_t)^2}
= \kappa(M_t)^2.
$$
Thus, virtualization improves conditioning exactly when \(M_t\) is close to the identity.

Recall from Theorem~B.4 that the mismatch matrix is defined as
$$
E_t := I - C\,\widehat C_t^{-1}.
$$
Equivalently,
$$
M_t = C\,\widehat C_t^{-1} = I - E_t.
$$
The next lemma shows that small mismatch implies \(\kappa(M_t)\) is close to \(1\).

\begin{lemma}[Condition number bound for the preconditioned system)]
\label{lem:B8_precond_kappa}
Let
$$
M_t := C\,\widehat C_t^{-1}
\qquad\text{and}\qquad
E_t := I - M_t.
$$
If
$$
\|E_t\|_2 < 1,
$$
then \(M_t\) is nonsingular and
$$
\kappa(M_t) \le \frac{1+\|E_t\|_2}{1-\|E_t\|_2}.
$$
Consequently,
$$
\kappa(M_t^\top M_t) \le \left(\frac{1+\|E_t\|_2}{1-\|E_t\|_2}\right)^2.
$$
\end{lemma}

\begin{proof}
We use the singular value characterizations
$$
\sigma_{\max}(M_t) = \sup_{\|x\|_2=1}\|M_t x\|_2,
\qquad
\sigma_{\min}(M_t) = \inf_{\|x\|_2=1}\|M_t x\|_2.
$$
Since \(M_t = I - E_t\), for any \(x\) with \(\|x\|_2=1\),
$$
\|M_t x\|_2 = \|(I-E_t)x\|_2 = \|x - E_t x\|_2.
$$

\paragraph{Upper bound.}
By the triangle inequality,
$$
\|x - E_t x\|_2 \le \|x\|_2 + \|E_t x\|_2 = 1 + \|E_t x\|_2.
$$
By the definition of \(\|E_t\|_2\),
$$
\|E_t x\|_2 \le \|E_t\|_2\|x\|_2 = \|E_t\|_2.
$$
Therefore \(\|M_t x\|_2 \le 1+\|E_t\|_2\) for all unit vectors \(x\), and taking the supremum gives
$$
\sigma_{\max}(M_t) \le 1+\|E_t\|_2.
$$

\paragraph{Lower bound.}
By the reverse triangle inequality \(\|a-b\|_2 \ge \big|\|a\|_2-\|b\|_2\big|\), we have
$$
\|x - E_t x\|_2 \ge \big|\|x\|_2 - \|E_t x\|_2\big| = 1 - \|E_t x\|_2.
$$
Using again \(\|E_t x\|_2 \le \|E_t\|_2\|x\|_2=\|E_t\|_2\), we obtain
$$
\|M_t x\|_2 \ge 1-\|E_t\|_2
\qquad\text{for all }\|x\|_2=1.
$$
Taking the infimum over unit vectors yields
$$
\sigma_{\min}(M_t) \ge 1-\|E_t\|_2.
$$
Because \(\|E_t\|_2<1\), the right-hand side is strictly positive, hence \(\sigma_{\min}(M_t)>0\) and \(M_t\) is nonsingular.

\paragraph{Condition number bound.}
Finally,
$$
\kappa(M_t) = \frac{\sigma_{\max}(M_t)}{\sigma_{\min}(M_t)}
\le \frac{1+\|E_t\|_2}{1-\|E_t\|_2}.
$$
Squaring both sides and using \(\kappa(M_t^\top M_t)=\kappa(M_t)^2\) gives the last claim.
\end{proof}

\subsubsection{Kalman filtering shrinks the mismatch bound}
QADAPT treats the unknown couplings as latent variables and uses a Kalman filter to refine them from CSD-derived observations.
We formalize the contraction improvement in the standard linear-Gaussian setting.

Let \(c:=\mathrm{vec}(C)\in\mathbb{R}^{D^2}\) be the vectorized coupling parameters and suppose that at each step \(t\) we obtain
an observation \(y_t\in\mathbb{R}^m\) of the form
$$
y_t = H_t c + \nu_t,
\qquad
\nu_t \sim \mathcal{N}(0,R_t),
\qquad
R_t \succ 0.
$$

Let \((\hat c_t,P_t)\) be the Kalman posterior mean and covariance after assimilating \(y_{1:t}\).
In information form, the covariance satisfies
$$
P_t^{-1} = P_{t-1}^{-1} + H_t^\top R_t^{-1}H_t.
$$

\begin{lemma}[Posterior covariance is monotone decreasing]
\label{lem:cov_monotone}
For all \(t\), \(P_t \preceq P_{t-1}\).
\end{lemma}

\begin{proof}
Since \(R_t\succ 0\), the matrix \(H_t^\top R_t^{-1}H_t\) is positive semidefinite, so
$$
P_t^{-1} - P_{t-1}^{-1} = H_t^\top R_t^{-1}H_t \succeq 0,
\qquad\text{i.e.}\qquad
P_t^{-1} \succeq P_{t-1}^{-1}.
$$
We now use the implication \(A\succeq B \succ 0 \Rightarrow A^{-1}\preceq B^{-1}\).
To prove it: let \(B^{1/2}\) be the PSD square root. Then \(A\succeq B\) implies
$$
B^{-1/2}AB^{-1/2} \succeq I.
$$
Inverting both sides (order reverses under inversion for PSD matrices) yields
$$
(B^{-1/2}AB^{-1/2})^{-1} \preceq I,
$$
and multiplying by \(B^{-1/2}\) on left and right gives \(A^{-1}\preceq B^{-1}\).
Applying this with \(A=P_t^{-1}\) and \(B=P_{t-1}^{-1}\) yields \(P_t\preceq P_{t-1}\).
\end{proof}

Write \(\widehat C_t := \mathrm{unvec}(\hat c_t)\) and \(\Delta_t := \widehat C_t - C\).
Conditioned on \(y_{1:t}\), we have \(c-\hat c_t \sim \mathcal{N}(0,P_t)\).
Thus \(\mathrm{vec}(\Delta_t) \sim \mathcal{N}(0,P_t)\) and
$$
\mathbb{E}\|\Delta_t\|_F^2 = \mathbb{E}\|\mathrm{vec}(\Delta_t)\|_2^2 = \mathrm{tr}(P_t).
$$

\paragraph{High-probability bound}
Let \(n:=D^2\). Let \(g\sim\mathcal{N}(0,I_n)\) and write \(\mathrm{vec}(\Delta_t)=P_t^{1/2}g\).
A standard chi-square concentration inequality (Laurent--Massart)~\cite{massart2007concentration} gives that for any \(s>0\),
$$
\mathbb{P}\!\left(\|g\|_2^2 \ge n + 2\sqrt{ns} + 2s\right) \le e^{-s}.
$$
Setting \(s=\log(1/\delta)\) yields: with probability at least \(1-\delta\),
$$
\|g\|_2 \le \sqrt{n + 2\sqrt{n\log(1/\delta)} + 2\log(1/\delta)}.
$$
Moreover,
$$
\|\Delta_t\|_F = \|\mathrm{vec}(\Delta_t)\|_2 = \|P_t^{1/2}g\|_2 \le \|P_t^{1/2}\|_2\,\|g\|_2 = \sqrt{\|P_t\|_2}\,\|g\|_2,
$$
where the inequality follows from the definition of \(\|P_t^{1/2}\|_2\).
Therefore, with probability at least \(1-\delta\),
$$
\|\Delta_t\|_2 \le \|\Delta_t\|_F
\le \sqrt{\|P_t\|_2}\,\sqrt{n + 2\sqrt{n\log(1/\delta)} + 2\log(1/\delta)}.
$$

Combining this with Corollary~\ref{cor:contract_condition} yields an explicit sufficient condition on the Kalman covariance
for strict contraction (with high probability):
if
$$
\|C^{-1}\|_2 \sqrt{\|P_t\|_2}\,\sqrt{n + 2\sqrt{n\log(1/\delta)} + 2\log(1/\delta)} < \frac12,
$$
then \(\|E_t\|_2<1\) and Theorem~\ref{thm:ideal_cancel} gives geometric decay \(\|e_t\|_2\le\rho^t\|e_0\|_2\) on the linear model.

\subsubsection{Robustness: policy error and nonlinear remainder}
We now re-introduce (i) imperfect virtual actions and (ii) the Taylor remainder from Lemma~\ref{lem:taylor_remainder}.
Assume we pick
$$
\Delta v_t = -e_t + \eta_t
$$
for some action error \(\eta_t\), and keep the first-order model around \(u_t\).
Let \(C_t := J(u_t)\). By Lemma~\ref{lem:taylor_remainder} applied at \(u_t\), there exists \(r_t\) such that
$$
e_{t+1} = e_t + C_t\Delta u_t + r_t,
\qquad
\|r_t\|_2 \le \frac{L}{2}\|\Delta u_t\|_2^2.
$$
With \(\Delta u_t=\widehat C_t^{-1}\Delta v_t\), we obtain
$$
e_{t+1}
= \big(I - C_t\widehat C_t^{-1}\big)e_t + C_t\widehat C_t^{-1}\eta_t + r_t.
$$
Taking norms and using the definition of \(\|\cdot\|_2\) as in Lemma~\ref{lem:Et_delta} yields the explicit inequality
$$
\|e_{t+1}\|_2
\le \|I-C_t\widehat C_t^{-1}\|_2\,\|e_t\|_2
   + \|C_t\widehat C_t^{-1}\|_2\,\|\eta_t\|_2
   + \frac{L}{2}\|\widehat C_t^{-1}\|_2^2 \,\|\Delta v_t\|_2^2.
$$
In particular, if \(\|I-C_t\widehat C_t^{-1}\|_2\le\rho<1\) uniformly and \(\eta_t\equiv 0\),
then since \(\|\Delta v_t\|_2=\|e_t\|_2\),
$$
\|e_{t+1}\|_2 \le \rho\|e_t\|_2 + \frac{L}{2}\|\widehat C_t^{-1}\|_2^2\|e_t\|_2^2,
$$
which implies strict contraction for all \(\|e_t\|_2\) in a sufficiently small neighborhood (standard local-stability argument).

\begin{proposition}[Objective-level factorization induced by virtualization]
Let
\[
J_t(\Delta u)=\frac12 \|e_t + C_t \Delta u\|_2^2,
\qquad
e_t := x(u_t)-x^\star,
\qquad
C_t:=\nabla x(u_t).
\]
Then:

(i) \(J_t\) is separable across agents, i.e.
\[
J_t(\Delta u)=c_t+\sum_{i=1}^D f_{t,i}(\Delta u_i),
\]
if and only if \(C_t^\top C_t\) is diagonal.

(ii) If \(C_t\) is invertible and we use exact virtual coordinates
\[
\Delta u=C_t^{-1}\Delta v,
\]
then
\[
J_t(\Delta v)=\frac12 \sum_{i=1}^D (e_{t,i}+\Delta v_i)^2,
\]
so the one-step control objective is exactly separable.

(iii) With an estimate \(\hat C_t\), define
\[
E_t := I - C_t \hat C_t^{-1}.
\]
Then
\[
J_t^{\mathrm{virt}}(\Delta v)
=
\frac12\|e_t + (I-E_t)\Delta v\|_2^2.
\]
If
\[
\widetilde J_t(\Delta v):=\frac12\|e_t+\Delta v\|_2^2,
\]
then
\[
|J_t^{\mathrm{virt}}(\Delta v)-\widetilde J_t(\Delta v)|
\le
\|E_t\|_2 \|e_t+\Delta v\|_2 \|\Delta v\|_2
+\frac12 \|E_t\|_2^2 \|\Delta v\|_2^2.
\]
\end{proposition}

\begin{proof}
Expand
\[
J_t(\Delta u)
=
\frac12 e_t^\top e_t + (C_t^\top e_t)^\top \Delta u
+ \frac12 \Delta u^\top C_t^\top C_t \Delta u.
\]
Hence for \(i\neq j\),
\[
\frac{\partial^2 J_t}{\partial \Delta u_i \partial \Delta u_j}
=
(C_t^\top C_t)_{ij}.
\]
So all mixed partials vanish if and only if \(C_t^\top C_t\) is diagonal, which is equivalent to separability for this quadratic objective.

Now substitute \(\Delta u=C_t^{-1}\Delta v\):
\[
J_t(\Delta v)=\frac12 \|e_t+\Delta v\|_2^2
=
\frac12 \sum_i (e_{t,i}+\Delta v_i)^2.
\]

For the approximate case, write \(M_t=C_t\hat C_t^{-1}=I-E_t\), so
\[
J_t^{\mathrm{virt}}(\Delta v)
=
\frac12\|e_t + M_t\Delta v\|_2^2.
\]
Then
\[
J_t^{\mathrm{virt}}(\Delta v)-\widetilde J_t(\Delta v)
=
-\langle e_t+\Delta v,\; E_t\Delta v\rangle
+\frac12 \|E_t\Delta v\|_2^2,
\]
and the bound follows from Cauchy--Schwarz and
\[
\|E_t\Delta v\|_2 \le \|E_t\|_2\|\Delta v\|_2.
\]
\end{proof}

This shows that other MARL techniques such as MADDPG, QMIX, and FACMAC learn coordination in a fixed action space through centralized critics or value mixing. QADAPT instead tries to make the control problem itself additive by learning a physically grounded change of coordinates.

\newpage 

\section{Details on the Simulated Quantum Dot Environment}
\label{app:params}

In our simulator, quantum dot array configurations are randomized to ensure a broad and realistic training environment, including noise strengths, electrostatic couplings, and lever arms. We also reduce the contrast of charge stability diagram features toward the edges of our measurement environment, mimicking real devices whereby highly occupied quantum dots crossover into a classical transport regime and transition lines vanish. By accounting for such physical behaviors, we aim to bridge the sim-2-real gap of device tuning. 

\begin{table*}[h]
\caption{Details on the base QArray quantum dot simulator. When a min and a max value are specified, we sample uniformly at random between those limits at each reset of the environment. Cxy denotes the cross-capacitance matrix between potentials of type x and of type y, where d = dot potentials, g = plunger gates, b = barrier gates.}
\label{tab:parameters}
\centering
\small
\setlength{\tabcolsep}{6pt}
\begin{tabular}{lccc}
\toprule
Parameter & Min & Max & Value \\
\midrule
Cdd nearest coupling & 0 & 0.2 & - \\
Cdd second nearest coupling & 0 & 0.1 & - \\
Cgd lever arm & 0.95 & 1.0 & -\\
Cgd nearest coupling & 0.3 & 0.7 & - \\
Cgd second nearest coupling & 0.01 & 0.3 & - \\
Cgd third nearest coupling & 0 & 0.01 & - \\
Dot to sensor coupling & 0.035 & 0.05 & - \\
Cbd nearest coupling & 0.04 & 0.08 & - \\
Cbd second nearest coupling & 0.01 & 0.03 & - \\
Cbd third nearest coupling & 0.005 & 0.015 & - \\
Cbg nearest coupling & 0.08 & 0.15 & - \\
Cbg second nearest coupling & 0.03 & 0.18 & - \\
Cbg third nearest coupling & 0.01 & 0.03 & - \\
Cbb nearest coupling & 0.03 & 0.08 & - \\
Cbb second nearest neighbour & 0.01 & 0.03 & - \\
Cbb third nearest neighbour & 0.005 & 0.015 & - \\
Barrier to sensor coupling & 0.0003 & 0.001 & - \\
Telegraph noise transition probability & 0 & 0.01 & - \\
Telegraph noise amplitude & 0 & 0.012 \\
Latching noise lead coupling probability & 0.2 & 1.0 & - \\
Latching noise inter-dot coupling probability & 0.2  & 1.0 & - \\
Temperature (mK) & 50 & 200 & - \\
Coulomb peak width & 0 & 0.4 & - \\
Base tunnel coupling & 0.5 & 3.0 & - \\
Tunnel coupling exponential factor ($\alpha$) & 0.0001 & 0.0008 & - \\
Max charge carriers & - & - & 4 \\
Optimal plunger dot charge occupancy & - & - & 1 \\
Optimal sensor charge occupancy (fractional) & - & - & 0.53 \\
\bottomrule
\end{tabular}
\end{table*}

\begin{table*}[h]
\caption{Details on the gym environment. When a min and a max value are specified, we sample uniformly at random between those limits at each reset of the environment. Distances are in units of simulator volts, which has dimensions of voltage but scaled up to equal the transition line spacing.}
\label{tab:parameters}
\centering
\small
\setlength{\tabcolsep}{6pt}
\begin{tabular}{lccc}
\toprule
Parameter & Min & Max & Value \\
\midrule
Total plunger voltage range & 80.0 & 100.0 & - \\
Scan size & 3.0 & 4.0 & - \\
Total barrier voltage range & 20.0 & 30.0 & - \\
Lower radius for zero additive white noise & 20.0 & 30.0 & - \\
Distance to maximum additive white noise & 5.0 & 10.0 & - \\
Distance to total white noise & 30.0 & 40.0 & - \\
Maximum additive white noise amplitude & - & - & 0.05 \\
Distance to zero reward & - & - & 40.0 \\
Distance to half reward & - & - & 1.0 \\
Distance to zero barrier reward & - & - & 6.0 \\
\bottomrule
\end{tabular}
\end{table*}

\section{DreamerV3 Setup}

We include DreamerV3~\cite{dreamer} as a representative modern model-based reinforcement learning baseline.
DreamerV3 learns a latent ``world model'' from experience and optimizes a policy by imagining trajectories in that latent space. In our setting, each environment step corresponds to a full measurement cycle
(Section~\ref{methods}), i.e., acquiring all neighbor-pair CSDs, updating the virtualization estimate,
then applying gate updates. DreamerV3 models the predictive distribution over next-step observations and rewards through a recurrent
state-space model (RSSM) operating on a learned latent state $z_t$.
In our case, the transition captures the effect of gate updates on subsequent measurement outcomes
(i.e., changes in CSD geometry and reward).

\section{Array Tuning Agent Rewards}

\begin{figure*}[ht]
  \vskip 0.2in
  \begin{center}
    \centerline{\includegraphics[width=0.8\textwidth]{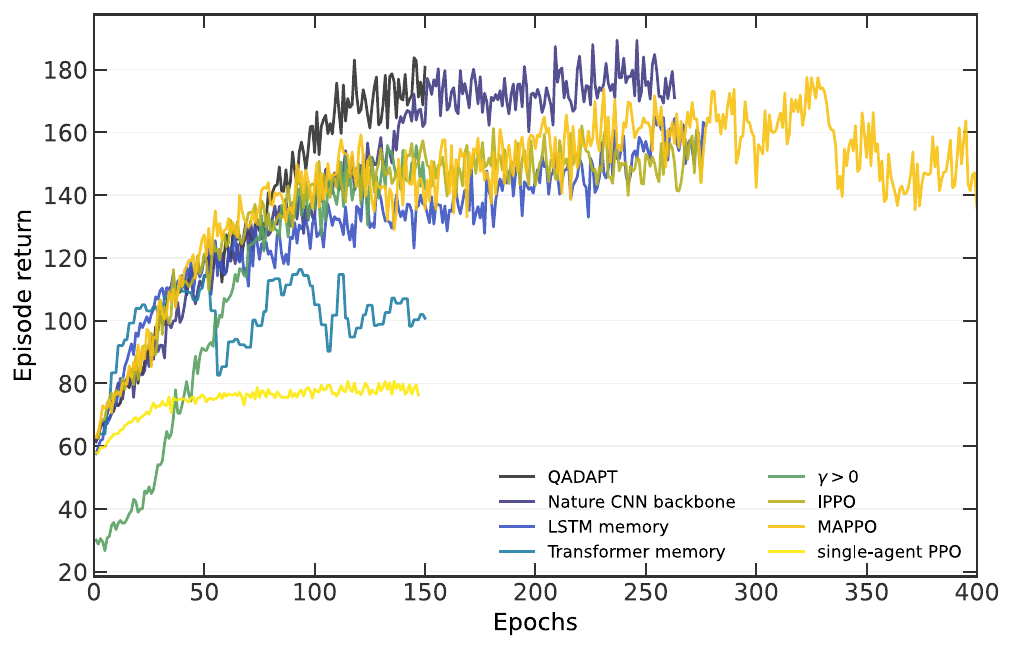}}
    \caption{Average per-agent reward for QADAPT trained on a 4-dot array, in addition to each version of the ablated base model, detailed in Section~\ref{sec:experiments}. In all benchmarks, trained parameters are always taken from the epoch in which models performed best.}
    \label{Fig4}
  \end{center}
\end{figure*}

\newpage

\section{Array Tuning Convergence Over Time}

\begin{figure*}[ht]
  \begin{center}
    \centerline{\includegraphics[width=\textwidth]{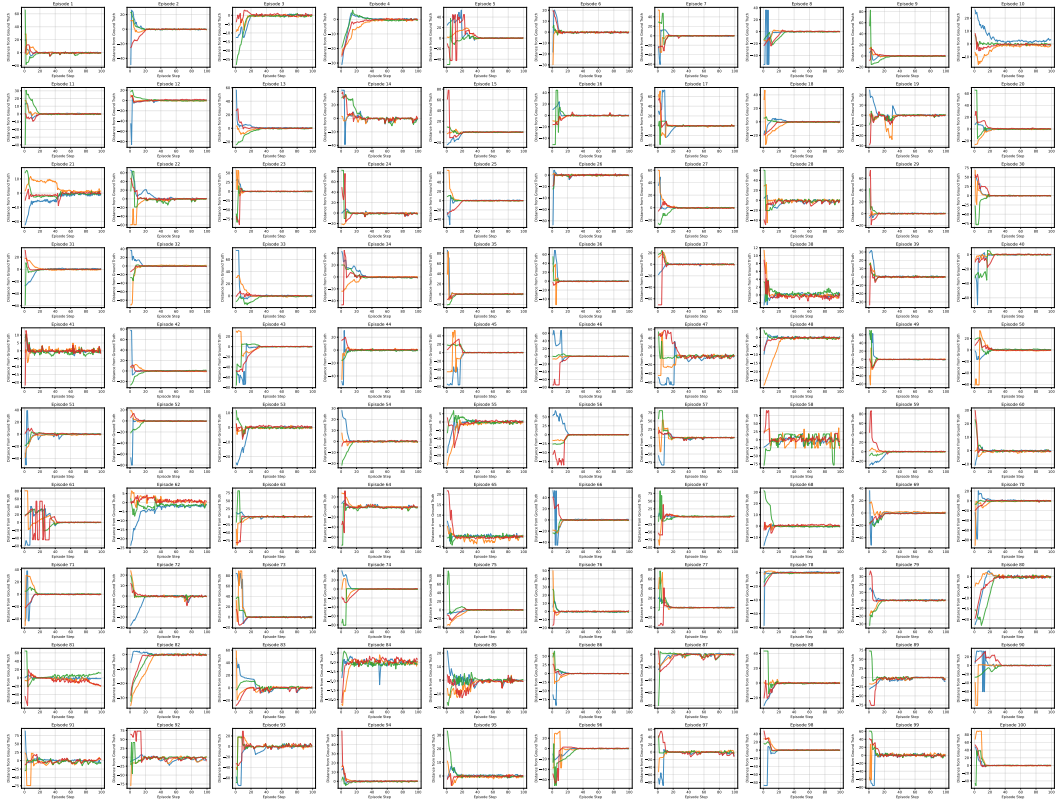}}
    \caption{Trained agent behavior sampled over 100 episodes for a 4-dot system. The four colors in each graph represent the signed distance to the ground truth position in voltage space for each of the plunger gate agents.}
    \label{Fig5}
  \end{center}
\end{figure*}

\newpage

\section{Training Hyperparameters and Compute Resources}

\begin{table*}[h]
\caption{Hyperparameters for PPO.}
\label{tab:parameters}
\centering
\small
\setlength{\tabcolsep}{6pt}
\begin{tabular}{lccc}
\toprule
Hyperparameter & Value \\
\midrule
Batch size & 16384 \\
yGrad clip & 40 \\
Minibatch size & 2048 \\
Epochs & 10 \\
Learning rate & 3e-5 \\
$\gamma$ & 0 \\
$\lambda$ & 0 \\
Clip & 0.2 \\
Entropy coefficient & 0.01 \\
Value function loss coeff & 0.5 \\
KL target & 0.01 \\
Predicted log std clip bounds & (-5, 2) \\
\bottomrule
\end{tabular}
\end{table*}

\section{Extension to Superconducting Qubits}
\label{app:super}

In this appendix, we apply QADAPT to the tuning of superconduting qubits in the presence of simultaneous drives. We largely follow the formulation of the control problem for flux-tunable transmon qubits, as laid out in \cite{wesdorp2026mitigating}; however, unlike in that work, we do not rely on analytical formulae and also extend the number of variables that are available for tuning. 

Conventionally, the tune-up of superconducting qubits is based on a series of experiments, each of which targets one or more tunable parameters to varying levels of precision. A good overview is provided in \cite{gao2021practical} and its references. Here, we consider five key parameters: bare qubit frequency, drive frequency, drive amplitude, phase, and a corrective envelope parameter (see Appendix~\ref{app:qubit_h} for full details). We probe the quality of qubits in response to these parameters using just one type of simulated experiment, known as an all-$XY$ calibration sequence \cite{reed2013entanglement}. In this protocol, 21 unique combinations of single-qubit gates are applied, and the resulting qubit state is measured. The result across all 21 operations ideally forms a ``staircase'' pattern where the qubit remains either entirely in the ground state, becomes fully excited, or enters into a perfect superposition of ground and excited states. The suboptimal calibration of any one parameter can lead to a deviation from the ideal staircase pattern. 

The sequential optimization of transmon qubit parameters does not guarantee good performance in a quantum processor due to the presence of drive-crosstalk. Furthermore, while individual parameter miscalibrations can lead to qualitatively different error syndromes, the all-$XY$ output can produce degenerate signatures when multiple parameters are far from their optimal configuration. These two effects together make transmon qubit tuning a challenge control problem, yet whose physical structure lends itself naturally to QADAPT.

We apply QADAPT to the transmon qubit tuning problem with minor modifications. Inputs, rather than being charge stability diagrams (CSDs) are now vectors of length 21. Actions are qubit parameter updates, which are split among two agent types, realized via shared policies -- one type controls $\{ \omega_{01}, \omega_d, \phi \}$ and the other controls $\{\Omega, \beta \}$. Refer to Appendix~\ref{app:qubit_h} for parameter descriptions and Appendix~\ref{app:trans_param} for their ranges. This parameter split is motivated by the physical similarities in the errors these parameters produce. For simplicity, we factor the action space directly based on the coupling between parameter adjustments and all-$XY$ outputs (see Appendix~\ref{app:super_sims}). The resulting cross-talk matrix is refined online at each time step. The tuning results are shown Figs.~\ref{Fig6} \& \ref{Fig7}.

\begin{figure*}[ht]
  \begin{center}
    \centerline{\includegraphics[width=\textwidth]{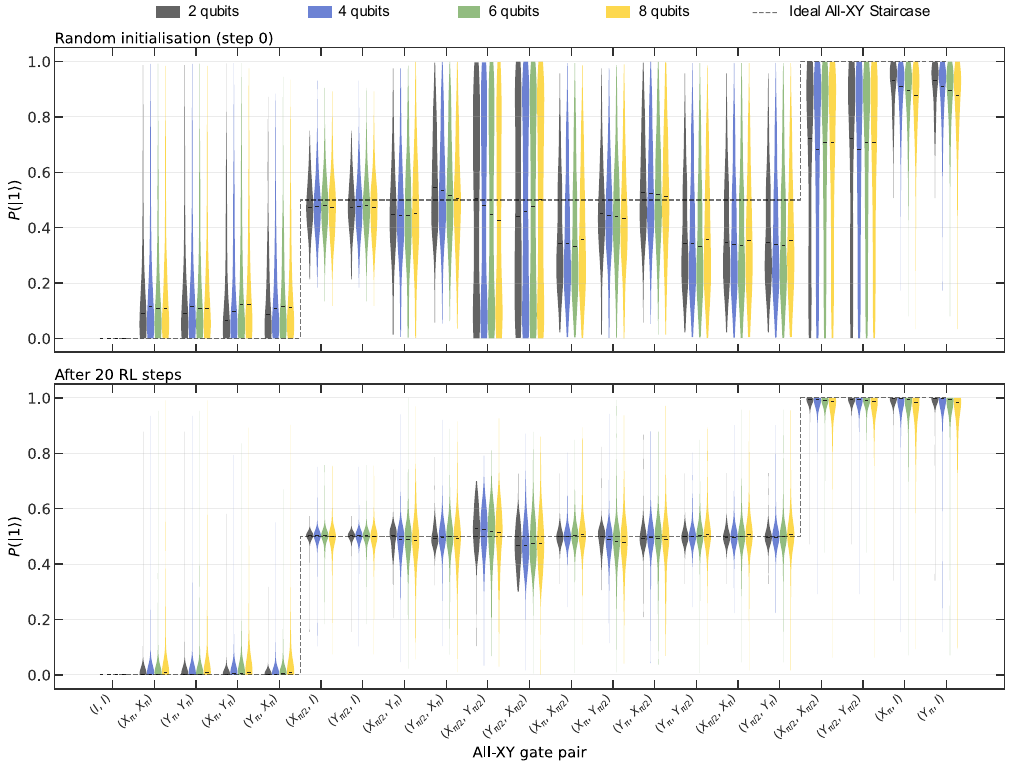}}
    \caption{\textbf{All-$XY$ Sequence Calibrated using QADAPT.} After random initialization (top), QADAPT reliably tunes superconducting qubit parameters (bare qubit frequency, drive frequency, amplitude, phase, and DRAG coefficient) to near-optimal values in 20 steps (bottom). All violin plots represent distributions over 100 tuning runs, and colors correspond to various qubit counts. The same policy, trained on a four-qubit system, is used across all shown qubit counts.}
    \label{Fig6}
  \end{center}
\end{figure*}

\begin{figure*}[ht]
  \begin{center}
    \centerline{\includegraphics[width=\textwidth]{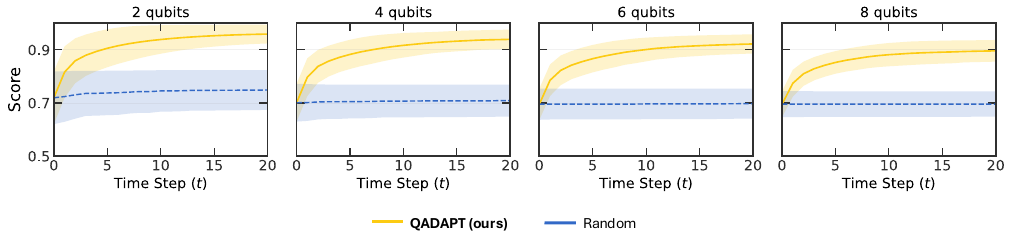}}
    \caption{\textbf{Convergence of Superconducting Qubit Tuning}. QADAPT generalizes well to the superconducting setting, also scaling zero-shot to qubit numbers of 2, 6 and 8 using a policy trained on 4 qubits. The score is defined as the normalized overlap with the target excited state probability, averaged over qubits and gate-pairs in the all-$XY$ calibration sequence. Lines represent averages over 100 runs, with an envelope of half a standard deviation. Initial parameter configurations are randomized in each run.}
    \label{Fig7}
  \end{center}
\end{figure*}

\newpage
\section{Transmon Qubit Hamiltonian}
\label{app:qubit_h}

The Hamiltonian of a transmon qubit $i$ in the presence of drive cross-talk from $n$ other, simultaneously driven, qubits, is given by:
\begin{equation}
\label{eq:hamiltonian}
    H = \omega_{01}^ia^\dagger a + \frac{\alpha^i}{2}a^\dagger a^\dagger a a + \mathrm{i}\sum_{j=1}^{n+1} \lambda_{ij}v^{j}(t)(a^\dagger - a),
\end{equation}
where $a^\dagger$ and $a$ are the creation and annihilation operators in the subspace of qubit $i$, $\omega_{01}^i$ is the qubit's bare transition frequency, $\alpha^i = \omega_{12}^i - \omega_{01}^i$ is the qubit anharmonicity, $\lambda_{ij}$ is the relative cross-talk magnitude from qubit $j$ onto qubit $i$, and $v^j(t)$ is the applied driving field, modulated by complex driving envelopes according to
\begin{equation}
\label{eq:drive}
    v^j(t) = s_I^j(t) \cos{(\omega_d^jt+ \phi^j)} + s_Q^j(t) \sin{(\omega_d^jt + \phi^j)} \qquad \forall \quad j \in \{0, ..., n\}.
\end{equation}
Here, $\omega_d^j$ and $\phi^j$ denote the angular drive frequency and phase on qubit $j$, and $s_I^j(t)$ and $s_Q^j(t)$ are time-dependent in-phase and quadrature drive envelopes. We employ raised-cosine pulse envelopes and standard Derivative Reduction by Adiabatic Gate (DRAG) correction, which relate $s_I^j$ and $s_Q^j$ in the following way,
\begin{align}
    s_I^j(t) &= \frac{\Omega^j}{2} \left[ 1 - \cos\left(\frac{2\pi t}{t_g}\right) \right], \label{eq:envelopes1}\\
    s_Q^j(t) &= -\left( \frac{\beta^j}{\alpha^j} \right) \frac{\mathrm{d}s_I^j}{\mathrm{d}t} = -\left( \frac{\beta^j}{\alpha^j} \right) \cdot \frac{\Omega^j \pi}{t_g} \sin\left(\frac{2 \pi t}{t_g}\right) \label{eq:envelopes2}, \\
\end{align}
where $\Omega^j$ is the nominal Rabi rate in $\mathrm{rad/ns}$, $t_g$ is the gate time, and $\beta^j$ is the DRAG coefficient. 

\subsection{Hardware effects}

Experimentally, the drive amplitude in volts $V$ is related to the Rabi rate via $\Omega^j = \eta^j \cdot V$, where $\eta^j$ is the capacitive coupling strength between the transmon and its drive line. This coupling is predominantly determined by device geometry, which is fixed at fabrication, and pre-calibrated in earlier tuning stages. However, we account for imperfect calibration, or changes in line attenuation between qubits, by introducing a prefactor to $\Omega^j$. Adjustments to the raw $\Omega^j$ by our algorithm should therefore be interpreted as corrective updates to the voltage amplitude of the qubit drive's carrier envelope. 

Similarly, we introduce group and phase delays to mimic the behavior of pulses propagating through real cables, which in general are also unique for each qubit. In summary, this transforms $t$, $\phi^j$ and $\Omega^j$ parameters appearing in Eqs.~(\ref{eq:drive}\textup{--}\ref{eq:envelopes2}) as follows,
\begin{align}
    t &\rightarrow t + t_\delta^j, \\
    \phi^j &\rightarrow \phi^j + \phi_\delta^j, \\
    \Omega^j &\rightarrow A^j \Omega^j,
\end{align}
and prevents agent policies from collapsing toward trivial parameter settings. All parameter ranges used for training are summarized in Table.~\ref{tab:super_ranges}

We further note that, in practice, the transmon qubit frequency is tuned via an applied DC current $I_{\mathrm{bias}}^j$ through an on-chip bias line. This current produces a flux equal to $\Phi^j = M\cdot I_{\mathrm{bias}}^j$, which scales $\omega_{01}^j$ according to \cite{koch2007charge},
\begin{equation}
    \omega_{01}^j (\Phi^j) \propto \left| \cos{\left( \frac{\pi \Phi^j}{\Phi_0} \right)}\right|^{\frac{1}{2}},
\end{equation}
where $\Phi_0 = h/2e$ is the flux quantum. By applying parameter updates onto $\omega_{01}^j$ directly, phase accumulation during the simulated gate becomes linearized, which better conditions the action space for agent learning. The characterization of $\omega_{01}^j$ with respect to $I_{\mathrm{bias}}^j$ is routinely done in experimental settings, allowing a direct mapping to be made between the two variables at deployment. For fixed-frequency qubits, parameter updates to $\omega_{01}$ would drop out from the tuning algorithm.

The effect of flux noise and flux cross-talk between qubits, as well as associated changes in the anharmonicity $\alpha^j$ caused by second order effects upon varying $\Phi^j$, are outside the scope of this work. We further neglect the weak dependence of the Rabi rate $\Omega^j$ on the flux operating point. A more complete treatment would explicit parametrize both $\alpha^j$ and $A^j$ as functions of $\omega_{01}^j$, which could be pre-calibrated via Rabi experiments at multiple flux operating points. However, since our algorithm treats each round of all-$XY$ measurements independently, with no temporal credit assignment, the new effective Rabi rate would become implicitly absorbed into $A^j$ in a real experiment for each new operating point. The inability to track flux-dependent dynamics is therefore not a limitation of the model, but rather a consequence of the immediate-feedback reward structure.

\vfill

\section{Transmon Qubit Simulation Parameters}
\label{app:trans_param}

\begin{table*}[h]
\caption{Parameter ranges used in transmon qubit simulations. When a min and a max value are specified, we sample uniformly at random between those limits at each reset of the environment. For values sampled from a normal distribution, the mean and standard deviation are indicated.}
\label{tab:super_ranges}
\centering
\small
\setlength{\tabcolsep}{6pt}
\begin{tabular}{lccc}
\toprule
Parameter & Min & Max & Value \\
\midrule
Number of qubits & - & - & 4 \\
Fock space truncation & - & - & 3 \\
$\omega_{01} / 2\pi$ (GHz) & 4 & 6 & -\\
$\alpha / 2\pi$ (MHz) & $-$300 & $-$200 & - \\
$\lambda_{ii'}$ & 0.0 & 0.5 & - \\
$t_g$ (ns) & 16 & 24 & - \\
$\omega_d / 2\pi$ (GHz) & - & - & $\sim \mathcal{N}(\omega_{01} / 2\pi, 0.01)$ \\
$\phi$ (rad) & $-\pi$ & $\pi$ & - \\
$\Omega$ (rad/ns) & - & - & $ 2\pi/t_g (1+ \delta), \quad \delta \sim \mathcal{N}(0, 0.1)$ \\
$\beta$ & 0.0 & 1.5 & - \\
$t_\delta$ (ns) & $-$2 & 2 & - \\
$\phi_\delta$ (rad) & $-\pi$ & $\pi$ & - \\
$A$ & 0.85 & 1.15 & - \\
\bottomrule
\end{tabular}
\end{table*}

\vfill

\begin{table*}[h]
\caption{Parameter bounds for transmon qubit simulations. Min and max values are stated relative to their initialization at the start of each training episode.}
\label{tab:parameters}
\centering
\small
\setlength{\tabcolsep}{6pt}
\begin{tabular}{lccc}
\toprule
Parameter & Min & Max \\
\midrule
$\omega_{01} / 2\pi$ (MHz) & $-$150 & 0 \\
$\omega_d / 2\pi$ (MHz) & $\omega_{01} -$ 20 & $\omega_{01} +$ 20 \\
$\phi$ (rad) & $-\pi$ & $\pi$\\
$\Omega$ (rad/ns) & 0.8 $\times 2\pi/t_g$ & 1.2 $\times 2\pi/t_g$\\
$\beta$ & 0.0 & 1.5 \\
\bottomrule
\end{tabular}
\end{table*}

\newpage 

\section{All-$XY$ Calibration Sequence}

The 21 gate pairs that comprise the all-$XY$ calibration sequence are provided in Table~\ref{tab:allxy} together with their ideal excited state probabilities. The drive amplitude $\Omega^j$ enters the Hamiltonian linearly, such that the resulting rotation angle on qubit $j$ scales proportionally with $\Omega^j$. Likewise, the drive phase $\phi^j$ enters the Hamiltonian as a rigid rotation of the drive vector in the $IQ$-plane, and so the axis of rotation is determined directly by $\phi^j$, independently of the rotation angle. As a result, all gates in the calibration sequence can be expressed using amplitude scalings and/or phase offsets relative to a $X_\pi$ base pulse. In real hardware, phase offsets can implemented at no extra cost using programmable `virtual' phases.

\begin{table*}[h]
\caption{Gate pairs participating in the all-$XY$ calibration sequence. The qubit is assumed to be in the ground state $\ket{0}$ at the start of each gate-pair. Amplitude scalings and phase offsets are stated relative to a perfectly calibrated $X_\pi$ gate pulse.}
\label{tab:allxy}
\centering
\small
\setlength{\tabcolsep}{6pt}
\begin{tabular}{ccccc}
\toprule
Index & Gate pair & Amplitude scaling & Phase offset & $P\left( \ket{1} \right) $ \\
\midrule
1 & $I,I$ & 0.0, 0.0 & 0, 0 & 0.0 \\
2 & $X_\pi,X_\pi$ & 1.0, 1.0 & 0, 0 & 0.0 \\
3 & $Y_\pi,Y_\pi$ & 1.0, 1.0 & $\frac{\pi}{2}$, $\frac{\pi}{2}$ & 0.0 \\
4 & $X_\pi,Y_\pi$ & 1.0, 1.0 & 0, $\frac{\pi}{2}$ & 0.0 \\
5 & $Y_\pi,X_\pi$ & 1.0, 1.0 & $\frac{\pi}{2}$, 0 & 0.0 \\
6 & $X_{\frac{\pi}{2}},I$ & 0.5, 0.0 & 0, 0 & 0.5 \\
7 & $Y_{\frac{\pi}{2}},I$ & 0.5, 0.0 & $\frac{\pi}{2}$, 0 & 0.5 \\
8 & $X_{\frac{\pi}{2}},Y_\pi$ & 0.5, 1.0 & 0, $\frac{\pi}{2}$ & 0.5 \\
9 & $Y_{\frac{\pi}{2}},X_\pi$ & 0.5, 1.0 & $\frac{\pi}{2}$, 0 & 0.5 \\
10 & $X_{\frac{\pi}{2}},Y_{\frac{\pi}{2}}$ & 0.5, 0.5 & 0, $\frac{\pi}{2}$ & 0.5 \\
11 & $Y_{\frac{\pi}{2}},X_{\frac{\pi}{2}}$ & 0.5, 0.5 & $\frac{\pi}{2}$, 0 & 0.5 \\
12 & $X_\pi,X_{\frac{\pi}{2}}$ & 1.0, 0.5 & 0, 0 & 0.5 \\
13 & $X_\pi,Y_{\frac{\pi}{2}}$ & 1.0, 0.5 & 0, $\frac{\pi}{2}$ & 0.5 \\
14 & $Y_\pi,X_{\frac{\pi}{2}}$ & 1.0, 0.5 & $\frac{\pi}{2}$, 0 & 0.5 \\
15 & $Y_\pi,Y_{\frac{\pi}{2}}$ & 1.0, 0.5 & $\frac{\pi}{2}$, $\frac{\pi}{2}$ & 0.5 \\
16 & $X_{\frac{\pi}{2}},X_\pi,$ & 0.5, 1.0 & 0, 0 & 0.5 \\
17 & $Y_{\frac{\pi}{2}},Y_\pi,$ & 0.5, 1.0 & $\frac{\pi}{2}$, $\frac{\pi}{2}$ & 0.5 \\
18 & $X_{\frac{\pi}{2}},X_{\frac{\pi}{2}}$ & 0.5, 0.5 & 0, 0 & 1.0 \\
19 & $Y_{\frac{\pi}{2}},Y_{\frac{\pi}{2}}$ & 0.5, 0.5 & $\frac{\pi}{2}$, $\frac{\pi}{2}$ & 1.0 \\
20 & $X_\pi,I$ & 1.0, 0.0 & 0, 0 & 1.0 \\
21 & $Y_\pi,I$ & 1.0, 0.0 & $\frac{\pi}{2}$, 0 & 1.0 \\
\bottomrule
\end{tabular}
\end{table*}

\newpage 

\section{Implementation Details of Qubit Dynamics}
\label{app:super_sims}

All transmon qubit simulations are performed by solving the time-dependent Schrödinger equation under the Hamiltonian in \ref{eq:hamiltonian} using \texttt{dynamiqs} \cite{guilmin2025dynamiqs}, a JAX-based quantum dynamics library. We rely heavily on the library's batching functionality to reduce compute time: all $N_\text{QUBITS} \times N_\text{AllXY} = 4 \times 21$ independent quantum systems are evolved simultaneously within a single \texttt{sesolve} call by constructing a batched Hamiltonian of shape $(N_\text{QUBITS}\times N_\text{AllXY}\times 3 \times 3)$, avoiding any Python-level loop over sequences or qubits. 

Each time-dependent drive is implemented as a scalar coefficient $f(t)$ that modulates the Hamiltonian operator $\mathrm{i}(a^\dagger - a)$. Cross-talk between drive lines is incorporated via a linear transformation of the per-qubit coefficients at each time step. $f$ is passed as a Python callable to \texttt{dq.modulated}, which \texttt{dynamiqs} evaluates on demand at each step taken by the adaptive ODE solver. To account for per-qubit timing offsets $t_\delta^j$, the integration window is expanded beyond $[0, 2t_g]$ to capture both earliest and latest arriving pulses, while each qubit's envelope is still evaluated in its own shifted time coordinate $\tau^j = t - t_\delta^j$, ensuring no pulse is truncated regardless of the sign of the delay.   

The full Jacobian of all-$XY$ staircase outcomes with respect to all qubit control parameters is of shape $(N_\text{QUBITS}\times N_\text{AllXY}\times N_\text{QUBITS}\times N_\text{params})$, and is also computed in a single forward pass using \texttt{jax.jacfwd}, which leverages \texttt{dynamiqs}' forward sensitivity mode \texttt{dq.gradient.Forward}. This mode offer a speed advantage when the number of parameters is fewer than the number of outputs, which is true in our case. The resulting Jacobian blocks are used to construct the cross-talk compensation tensor via pseudoinversion.

\section{Compute Resources and Cost}

We iterated on various modifications of QADAPT on a cluster of NVIDIA RTX A4000 GPUs with 20 GB of GPU memory. All our hero runs were performed on NVIDIA 8xH100 GPUs with 80 GB of GPU memory. The cost of these hero runs add up to 2500 USD on the modal platform.    

\section{Code Availability}

All the code for QADAPT is available through the anonymized url: \url{https://anonymous.4open.science/r/rl-agent-for-qubit-array-tuning-71D1/README.md}. We present the modal entry points to run this code as well.


\end{document}